\documentclass{article}

\usepackage{amsmath}  %
\usepackage{booktabs} %
\usepackage{xcolor}   %
\usepackage{graphicx} %
\usepackage{array}    %
\usepackage{multirow} %
\usepackage{makecell} %
\usepackage{caption}  %
\usepackage{colortbl} %
\usepackage{iclr2025_conference}
\usepackage[utf8]{inputenc} %

\usepackage{graphicx}

\usepackage[T1]{fontenc}

\IfFileExists{headers/config/showoverfull.config}{
	\overfullrule=1cm
}{
}

\usepackage{marginnote}

\usepackage[backgroundcolor=none,linecolor=red,textsize=footnotesize]{todonotes}

\usepackage{etoolbox}

\newbool{includeappendix}
\setbool{includeappendix}{true} %
\IfFileExists{headers/config/noappendix.config}{
	\setbool{includeappendix}{false}
}{}

\newif\ifincludeappendixx
\ifbool{includeappendix}{
	\includeappendixxtrue
}{
	\includeappendixxfalse
}

\usepackage{xr} %
\usepackage{filecontents}

\ifbool{includeappendix}{}{
	\input{appendix-labels-loader}

	\externaldocument{appendix-labels}
}

\usepackage{acro} %

\usepackage{listings}

\usepackage{textcomp}

\usepackage{xcolor}

\usepackage[scaled=0.8]{beramono}

\definecolor{ckeyword}{HTML}{7F0055}
\definecolor{ccomment}{HTML}{3F7F5F}
\definecolor{cstring}{HTML}{2A0099}

\lstdefinestyle{numbers}{
	numbers=left,
	framexleftmargin=20pt,
	numberstyle=\tiny,
	firstnumber=auto,
	numbersep=1em,
	xleftmargin=2em
}

\lstdefinestyle{layout}{
	frame=none,
	captionpos=b,
}

\lstdefinestyle{comment-style}{
	morecomment=[l]//,
	morecomment=[s]{/*}{*/},
	commentstyle={\color{ccomment}\itshape},
}

\lstdefinestyle{string-style}{
	morestring=[b]",%
	morestring=[b]',%
	stringstyle={\color{cstring}},
	showstringspaces=false,%
}

\lstdefinestyle{keyword-style}{
	keywordstyle={\ttfamily\bfseries},
	morekeywords={
		function,
		constructor,
		int,
		bool,
		return,
		returns,
		uint
	},
	morekeywords = [2]{},
	keywordstyle = [2]{\text},
	sensitive=true,
}

\lstdefinestyle{input-encoding}{
	inputencoding=utf8,
	extendedchars=true,
	literate=
	{ℝ}{$\reals$}1%
	{→}{$\rightarrow$}1%
	{α}{$\alpha$}1%
	{β}{$\beta$}1%
	{λ}{$\lambda$}1%
	{θ}{$\theta$}1%
	{ϕ}{$\phi$}1%
}

\lstdefinestyle{escaping}{
	moredelim={**[is][\color{blue}]{\%}{\%}},
	escapechar=|,
	mathescape=true
}

\lstdefinestyle{default-style}{
	basicstyle=\fontencoding{T1}\ttfamily\footnotesize,
	style=numbers,
	style=layout,
	style=comment-style,
	style=string-style,
	style=keyword-style,
	style=input-encoding,
	style=escaping,
	tabsize=2,
	upquote=true
}

\lstdefinelanguage{BASIC}{
	language=C++,
	style=default-style
}[keywords,comments,strings]%

\usepackage{hyperref,times}
\usepackage{url}
\usepackage{amsthm}
\usepackage{thm-restate}
\usepackage{mathtools}
\usepackage{algorithm}
\usepackage[noend]{algpseudocode}
\usepackage{wrapfig}
\usepackage{graphicx}
\usepackage{xspace}
\usepackage{subcaption}
\usepackage{multicol}

\theoremstyle{plain}
\newtheorem{theorem}{Theorem}[section]

\theoremstyle{definition}

\newtheorem{assumption}[theorem]{Assumption}
\theoremstyle{remark}

\usepackage{amsmath,amsfonts,bm}

\def\eqref#1{equation~\ref{#1}}

\def\1{\bm{1}}

\def\vf{{\bm{f}}}

\def\vz{{\bm{z}}}

\def\mA{{\bm{A}}}

\def\mC{{\bm{C}}}

\def\mW{{\bm{W}}}
\def\mX{{\bm{X}}}
\def\mY{{\bm{Y}}}
\def\mZ{{\bm{Z}}}

\DeclareMathAlphabet{\mathsfit}{\encodingdefault}{\sfdefault}{m}{sl}
\SetMathAlphabet{\mathsfit}{bold}{\encodingdefault}{\sfdefault}{bx}{n}

\def\gB{{\mathcal{B}}}
\def\gC{{\mathcal{C}}}

\def\gG{{\mathcal{G}}}

\def\gN{{\mathcal{N}}}

\newcommand{\sigmoid}{\sigma}

\DeclareMathOperator{\rank}{rank}

\DeclareMathOperator{\RowSpan}{rowspan}
\DeclareMathOperator{\ColSpan}{colspan}

\DeclareMathOperator{\NullSpace}{null}
\DeclareMathOperator{\nullity}{nullity}

\newcommand{\grad}[1]{\ensuremath{\tfrac{\partial\mathcal{L}}{\partial #1}}}
\newcommand{\bv}{\mathcal{V}}

\newcommand{\metric}[1]{GSM-#1}

\newcommand{\tool}{GRAIN\xspace}
\newcommand{\toollong}{\emph{\textbf{G}raph \textbf{R}econstruction \textbf{A}lgorithm for \textbf{In}version of Gradients}\xspace}

\newcommand{\dist}[1]{\text{dist}{(#1)}}
\newcommand{\glue}[1]{\text{glue}{(#1)}}
\newcommand{\dang}[1]{\text{dang}{(#1)}}
\newcommand{\overlap}[1]{\text{overlap}{(#1)}}
\newcommand{\ext}[1]{\text{ext}{(#1)}}
\definecolor{darkspringgreen}{rgb}{0.09, 0.45, 0.27}

\usepackage[capitalize]{cleveref}

\makeatletter
\renewcommand\theHALG@line{\thealgorithm.\arabic{ALG@line}}
\makeatother

\newcommand{\appref}[1]{%
	\ifbool{includeappendix}{\cref{#1}}{the appendix}%
}
\newcommand{\Appref}[1]{%
	\ifbool{includeappendix}{\cref{#1}}{The appendix}%
}

\crefname{equation}{Eq.}{Eq.}
\creflabelformat{equation}{#2#1#3}
\crefname{assumption}{Asm.}{Asm.}
\creflabelformat{assumption}{#2#1#3}
\crefname{appendix}{App.}{App.}
\crefname{section}{Sec.}{Sec.}
\crefname{theorem}{Thm.}{Thm.}
\crefname{algorithm}{Alg.}{Alg.}
\crefname{corollary}{Cor.}{Cor.}
\crefname{table}{Tab.}{Tab.}

\title{GRAIN: Exact Graph Reconstruction from Gradients}

\author{Maria Drencheva$^{1}$, Ivo Petrov$^{1}$, Maximilian Baader$^{2}$, Dimitar I. Dimitrov$^{1,2}$, Martin Vechev$^{1,2}$\hfil\\
	\hspace{5em}
	$^{1}$ INSAIT, Sofia University "St. Kliment Ohridski"
	\hspace{2em}
	$^{2}$ ETH Zurich
	\hfil\\
	\texttt{\{maria.drencheva, ivo.petrov, dimitar.iliev.dimitrov\}@insait.ai} $^{1}$\hfil\\
	\texttt{\{mbaader,martin.vechev\}@inf.ethz.ch}  $^{2}$\hfil\\
}

\iclrfinalcopy %
\begin{document}
	
\maketitle

\vspace{-6mm}
\begin{abstract}
Federated learning claims to enable collaborative model training among multiple clients with data privacy by transmitting gradient updates instead of the actual client data. However, recent studies have shown the client privacy is still at risk due to the, so called, gradient inversion attacks which can precisely reconstruct clients' text and image data from the shared gradient updates. While these attacks demonstrate severe privacy risks for certain domains and architectures, the vulnerability of other commonly-used data types, such as graph-structured data, remain under-explored. To bridge this gap, we present \tool{}, the first exact gradient inversion attack on graph data in the honest-but-curious setting that recovers both the structure of the graph and the associated node features. Concretely, we focus on Graph Convolutional Networks (GCN) and Graph Attention Networks (GAT) -- two of the most widely used frameworks for learning on graphs. Our method first utilizes the low-rank structure of GNN gradients to efficiently reconstruct and filter the client subgraphs which are then joined to complete the input graph. We evaluate our approach on molecular, citation, and social network datasets using our novel metric. 
We show that \tool{} reconstructs up to $80\%$ of all graphs exactly, significantly outperforming the baseline, which achieves up to 20\% correctly positioned nodes.
\vspace{-3mm}
\end{abstract}

\section{Introduction}
Graph Neural Networks (GNNs)~\citep{gnn} have shown a great promise in learning on graph-structured data like social networks, traffic flows, molecules, as well as healthcare and income data. Many of these applications, however, require large quantities of private data, which can be hard to collect due to privacy regulations and the reluctance of users to share their data due to fear of losing competitive advantage. This has naturally led to widespread use of GCNs and GATs alongside Federated Learning (FL) which promises to protect the sensitive data of users~\citep{xie2021federated, zhang2021fastgnn, zhu2022federated, lee2022privacy, lou2021stfl, peng2022fedni}. 

However, the privacy of client data in FL in different domains including images~\citep{zhang2023generative}, text~\citep{petrov2024dager}, and tabular data~\citep{vero2023tableak} was recently severely violated by the introduction of gradient inversion attacks in the honest-but-curious setting. In these attacks, the FL server infers the client data based on passively observed client gradients and the models where they were computed. 
However, no prior work investigated the vulnerability of GNNs to such attacks.

\paragraph{This work: Gradient inversion attack on graphs} In this work, we introduce the first gradient inversion attack on graphs called \toollong(\textbf{\tool}), specifically designed to attack GNNs by recovering both the graph structure and the node features. 
At the core of \tool is an efficient filtering mechanism to correctly identify likely subgraphs, which are then combined to reconstruct the entire graph.
In particular, we leverage span checks to exploit the rank-deficiency of GNN layer updates and recover both the discrete set of per-layer node features and the subgraph adjacency matrices. 
We then reconstruct the client input using a depth-first search (DFS) algorithm to piece together the full graph from the recovered subgraphs.

We evaluate our attack on real-world chemical, citation and social network datasets, achieving reconstruction accuracies of up to 75\% (exact) and 85\% (partial) on chemical graph classification, 61\% in citation graph classification, and 66\% in molecular node classification tasks with known node labels. Finally, we demonstrate that data-dependent traversal strategies allow \tool to scale to significantly larger graphs, recovering 85\% of graphs with around 25 nodes.

\paragraph{Main Contributions} Our main contributions are:
\begin{itemize}	
    \item The first gradient inversion attack on Graph Neural Networks, recovering both the graph structure and the node features. We provide an efficient implementation on \href{https://github.com/insait-institute/GRAIN}{GitHub}.\footnote{\href{https://github.com/insait-institute/GRAIN}{https://github.com/insait-institute/GRAIN}}
    \item A generalization of the theory presented by~\citet{petrov2024dager} facilitating an effective mechanism to recover individual feature vectors and thus enabling the use on GNN layers to recover the graph connectivity via efficient filtering. 
    \item A novel set of metrics for measuring the quality of the recovered client graph structure and node features, enabling the evaluation of graph gradient inversion attacks at scale.
    \item A thorough evaluation of \tool showing FL with GNNs does not preserve the client data privacy in realistic applications, as \tool often recovers clients' graph data exactly.
    
\end{itemize}
We believe this work is an important step to further quantify the risks of using private data in FL.

\section{Related Work}\label{sec:related}

Gradient inversion attacks~\citep{zhu2019deep}, are attacks to Federated Learning that aim to infer the client's private data from the FL updates clients share with the federated server. As such, they assume knowledge of the updates themselves, as well as the model weights on which the updates were computed. Depending on the attack model, gradient inversion attacks are either malicious~\citep{cah, rtf, decepticons, panning, fishing} if the attacker can additionally manipulate the model weights sent to the clients, or honest-but-curious~\citep{zhu2019deep,analyticPhong, zhao2020idlg,geiping2020inverting, aaai, zhang2023generative, ggl,deng2021tag, lamp, spear, petrov2024dager, vero2023tableak} if the attack is executed passively by just observing model weights and updates.

In this work, we focus on the harder setting of honest-but-curious gradient inversion attacks. Most existing honest-but-curious attacks formulate gradient inversion as an optimization problem~\citep{zhao2020idlg,geiping2020inverting,nvidia,aaai, zhang2023generative, ggl,deng2021tag, lamp}  where the attacker tries to obtain the data which corresponds to a client update that matches the observed one best. While this approach is effective in many domains like images~\citep{geiping2020inverting,nvidia,aaai, zhang2023generative, ggl} where the client data is continuous, it has been shown that the associated optimization problem is much harder to solve for domains where client inputs are discrete. Some prior works have attempted to alleviate this issue by relying on various continuous relaxation~\citep{lamp,vero2023tableak} to the discrete optimization problem with some success. 

 In contrast to such approaches, recent research has demonstrated that exact gradient inversion is possible for both continuous~\citep{spear} and discrete inputs~\citep{petrov2024dager} in certain neural architectures. Notably, DAGER~\citep{petrov2024dager} showed that when dealing with a large but countable number of options for the client input data, the low-rank structure of gradient updates in fully connected layers can be leveraged to efficiently test all possibilities and identify the true input data. \tool{} extends this theory to GNN layers, exploiting the discrete nature of the unknown to the attacker adjacency matrix $\mA$ to simultaneously recover the client input features and graph structure, under the assumption of discrete input features. This addresses a critical challenge in graph-specific gradient inversion, where the interdependence between the recovery of the client feature matrix $\mX$ and the adjacency matrix $\mA$ renders traditional optimization-based attacks ineffective. Unlike DAGER, however, the structure of GNNs only allows for the recovery of the local graph structure using this approach. To overcome this, we further introduce a DFS-based algorithm that combines local graph structures into a single graph, enabling the recovery of the full client input data.

\section{Background and Notation}
\label{sec:background}

\paragraph{Threat Model}
	\tool{} is a honest-but-curious gradient inversion attack executed by a malicious FL server that aims to recover the clients' private data. As such, the server is assumed to know the weight updates sent to clients and the corresponding responses received from them. Following most existing gradient inversion attacks~\citep{deng2021tag,lamp,vero2023tableak,geiping2020inverting}, we also assume knowledge of the client data structure, including the semantic meaning, value ranges, and normalization of individual input features. This is well-justified as the server needs to enforce consistency in input representations across clients to ensure correct training.

As \tool{} represents the first gradient inversion attack on GNNs, it targets the most traditional federated protocol, FedSGD~\citep{mcmahan2017communication}, and the most commonly used GNN architectures --- GCN and GAT. Further, as \tool{} is based on our extension to~\cref{thm:spancheck_dager}, introduced by~\citet{petrov2024dager}, it makes two additional assumptions: (i) the number of nodes in the client graphs is smaller than the embedding dimension of the client GNN layers; and (ii) all input node features in the client graphs are discrete. We denote by $m$ the total number of features, by $\mathcal{F}_i$ the set of possible values for the $i$-th feature, and by $\mathcal{F} = \mathcal{F}_1 \times \mathcal{F}_2 \times \dots \times \mathcal{F}_m$ the set of all possible feature vectors. As we show in~\cref{sec:eval_top}, these assumptions cover many realistic use cases of GNNs.

\paragraph{Graph Terminology} Next, we introduce our graph notations. By $\bv$ we denote the set of possible graph nodes, where each node $v\in\bv$ is associated with a given feature vector. 
For an undirected graph $\mathcal{G} = (V, E)$ with node set $V\subset\bv$ of size $n = |V|$ and edge set $E$, we denote the degree of a node $v\in V$ with $\deg_\mathcal{G}(v)$. Further, for a pair of vertices $v_s, v_e \in V$, the distance $\dist{v_s, v_e}$ denotes the number of edges in the shortest path connecting $v_s$ to $v_e$. We introduce the notion of a $k$-hop neighborhood of a node $v$, defined by the subgraph $\mathcal{N}^k_\gG(v) = (V^k_v, E^k_v) \subset \mathcal{G}$ consisting of all nodes $V^k_v = \{v'\in V \mid \dist{v, v'} \leq k\}$ in the graph at a distance $\leq k$ from $v$ and the edges between them that can be traversed from $v$ in $\leq k$ steps $E^{k}_v = \{e=(v_1, v_2) \in E\mid v_1 \in V(\mathcal{N}^{k-1}_\gG(v)), v_2\in V^{k}_v\}$ with $\mathcal{N}^0_\gG(v)=\{v\}$. Finally, we will call a triplet $(V^k_v, E^k_v, v)$ associated with the $k$-hop neighborhood around $v$ in $\gG$ \emph{the building block $\gG_v^k$ with center $v$}.

\paragraph*{Graph Neural Networks}
Graph Neural Networks (GNNs) extend traditional neural networks to handle graph-structured data by leveraging the edges between nodes through message passing. Each GNN layer captures complex relationships between nodes by combining information from their neighbors and the graph's structural properties. This allows the model to learn richer node embeddings and gain insights into the graph's topology.
In particular, the $l^\text{th}$ GNN layer takes as an input a matrix $\mX^l \in \mathbb{R}^{n \times d}$ of $d$-dimensional node features for each node $v\in V$ and performs a combination of messages passing and non-linearity to produce the node features of the next layer $\mX^{l+1}$:
\begin{equation}
    \label{eq:gnn}
    \mX^{l+1} = \sigma(\mZ^l) = \sigma(\mA^l \mY^l) = \sigma\left(\mA^l \mX^l \mW^l \right),
\end{equation}
where $\mA^l \in \mathbb{R}^{n \times n}$ is a weighted adjacency matrix, $\mW^l \in \mathbb{R}^{d \times d'}$ is the weight matrix, $\mY^l \in \mathbb{R}^{n \times d'}$ is the output to the linear layer, and $\sigma$ is an activation function. For GCNs, the adjacency weights are calculated using the respective node degrees, while for GATs they are determined by the attention mechanism. We denote the input features matrix with $\mX^0$, and abuse the notation $\mX^l_v$ to denote the row of $\mX^l$ corresponding to the node $v\in V$. We consider $L$-layer GNNs, where we denote with $f_l$ the function that maps the input graph to the output of the $l^{\text{th}}$ layer for $l=0,1,2,\dots,L-1$.

\paragraph*{Gradient Filtering in Linear Layers}
Recently, DAGER~\citep{petrov2024dager} showed that one can leverage the gradients of the network loss $\mathcal{L}$ w.r.t. the weights $\mW^l$ of the $l^\text{th}$ linear layer $\grad{\mW^l}$ to search for the correct set of inputs $\mX^l$ to the layer among a discrete set of possibilities via filtering enabled by the low-rankness of the weight updates. We restate the theoretical findings below: 

\begin{theorem}
    \label{thm:spancheck_dager}
    If $n < d$ and if the matrix $\grad{\mY^l}$ is of full rank, then $\RowSpan(\mX^l) = \ColSpan(\grad{\mW^l})$.
\end{theorem}

To verify whether an input vector $\vz$ can be a part of the client input, DAGER performs a spancheck by measuring the distance between $\vz$ and the subspace spanned by the column vectors of $\grad{\mW^l}$:
\begin{equation*}
    d(\vz, \grad{\mW^l}) \coloneq \| \vz - \text{proj}(\vz, \ColSpan(\grad{\mW^l})) \|_2.
\end{equation*}
We say $\vz$ can be a part of the $l$-th layer input if $d(\vz, \grad{\mW^l}) < \tau$ for a chosen threshold $\tau$. In our work, we will extend \cref{thm:spancheck_dager} to \cref{eq:gnn}, in particular applying it to the linear layer $\mY^l = \mX^l\mW^l$.

\section{Overview of \tool}
\begin{figure}
	\vspace{-9mm}
	\includegraphics[width=\textwidth]{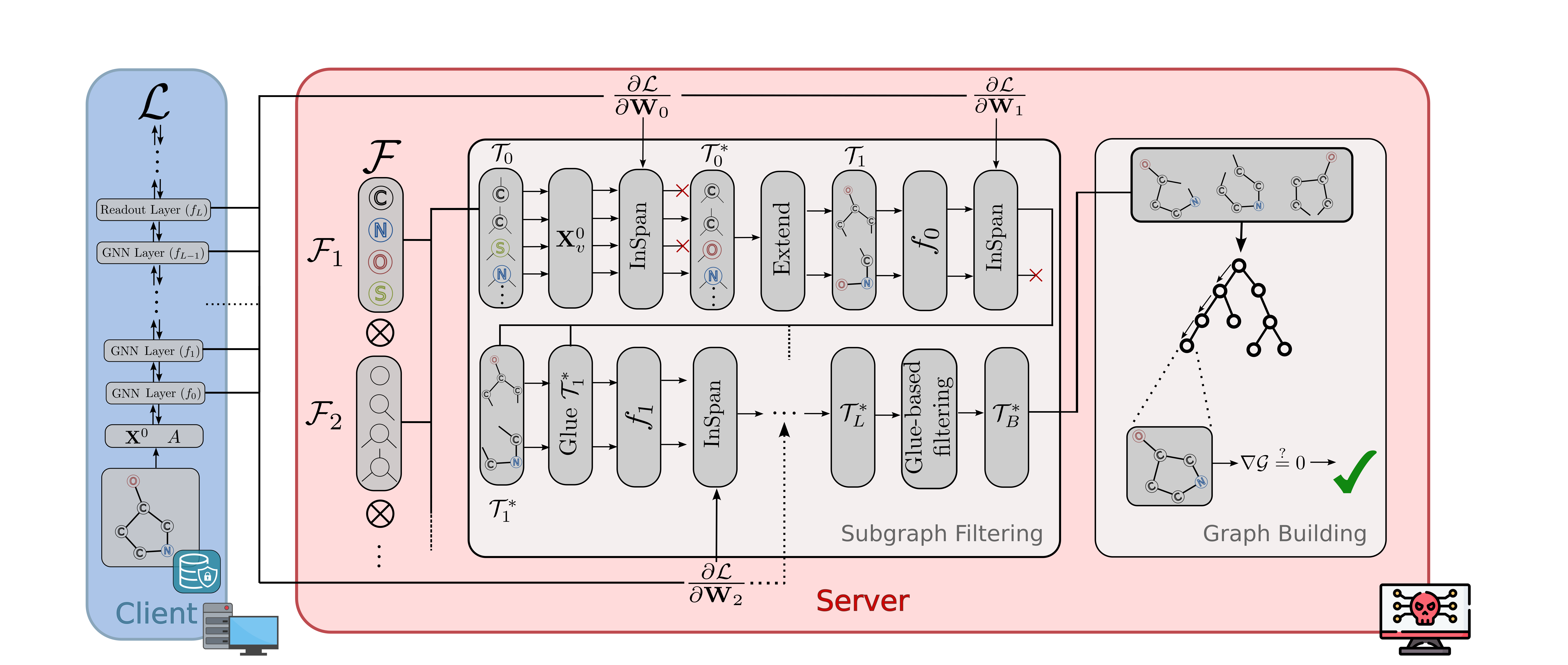}
	\vspace{-6mm}
	\caption{Overview of \tool . \tool first recovers the input nodes $\mathcal {T}_0^*$ by filtering through the cross-product $\mathcal {T}_0$ of all possible feature values, e.g., all atom types $\mathcal{F}_1$ and all number of bonds $\mathcal{F}_2$. It then iteratively combines and filters them into a set of larger building blocks $\mathcal {T}_B^*$ up to a degree $L$. Finally, it reconstructs the input graph by combining building blocks from $\mathcal {T}_B^*$ in a DFS manner.}	\label{fig:overview}
	\vspace{-4mm}
\end{figure}
Next, we provide a high-level overview of \tool, visualized in \cref{fig:overview}. \tool a gradient inversion attack designed to reconstruct graph-structured client training data in FL assuming an honest-but-curious adversary and is based on the key observation that due to the architecture of GNNs, the input embedding of a node $v$ at layer $l$ can be influenced only by the original input embeddings of the nodes in the $l$-hop neighborhood of $v$. Therefore, given any $l$-hop neighborhood, we can obtain the corresponding embedding of its center $v$ and use our spancheck, inspired by \cref{thm:spancheck_dager} and shown in \cref{thm:spancheck_ours}, to check whether the neighborhood is a possible subgraph of the input graph $\gG$ or not. 

\paragraph{Filtering} We leverage this subgraph checking procedure to create the filtering stage of \tool. In particular, we first generate the node proposal set $\mathcal{T}_0$ consisting of all nodes the can be obtained using the known to the attacker sets of possible feature values $\mathcal{F}_i$. We then apply the span checks layer-by-layer starting from $\mathcal{T}_0$. At each layer $l$, we "glue" the $(l-1)$-hop building blocks recovered from the previous filtering iteration into the set of possible $l$-hop neighbourhoods, denoted $\mathcal{T}_{l}$ in \cref{fig:overview}, which are then filtered using \cref{thm:spancheck_ours} to produce the set of consistent $l$-hop building blocks $\mathcal{T}_{l}^*$. Finally, at layer $L$ a final consistency check is performed to obtain the final set of $L$-hop building blocks $\mathcal{T}_B^*$.

\paragraph{Graph Building} In the second stage, we perform graph building where we combine the $L$-hop building blocks in $\mathcal{T}_B^*$ using a DFS-based approach to obtain the final graph reconstruction. To do this, at each node of the DFS tree, we "glue" a building block at a graph node that does not yet have enough neighbors to match their degree. Here, we use that the degree of a node is a widely used node feature for training GNNs 
\citep{hamilton2017inductive, xu2018powerful, 
	cui2022positional},
and is thus known by the attacker at this stage. When we cannot extend the graph further, we compute its gradient and compare it to the client gradient. If they do not match, we backtrack and try a different path. Otherwise, we terminate the DFS successfully and return the reconstructed graph.

\section{\tool : Exact Graph Reconstruction from Graidents}\label{sec:technical}
\begin{wrapfigure}[7]{r}{0.525\textwidth}
	\begin{minipage}{0.525\textwidth}
		\vspace{-9mm}
		\begin{algorithm}[H]
			\caption{The \tool algorithm}
			\label{alg:grain}
			\begin{algorithmic}[1]
				\Function{\tool}{$\mathcal{T}_{0}$, $\grad{\mW}$, $\tau$, $\vf$, $\mathcal{C}$} 
				\State $\mathcal{T}^{*}_{L}\gets$\Call{GenerateBBs}{$\mathcal{T}_{0}$, $\grad{\mW}$, $\tau$, $\vf$}

                \State $\mathcal{T}^*_B\gets$\Call{StructureFilter}{$\mathcal{T}_{L}^*$, $\grad{\mW}$}
				\State \Return \Call{ReconstructGraph}{$\mathcal{T}_{B}^*$, $\grad{\mW}$, $\mathcal{C}$}
				\EndFunction
			\end{algorithmic}
		\end{algorithm}
		\end{minipage}
\end{wrapfigure} 
We now present the technical details of \tool. 
First, in \cref{sec:gluing}, we explain the key operation of graph gluing. 
Then, in \cref{sec:spancheck}, we present \cref{thm:spancheck_ours} that adapts \cref{thm:spancheck_dager} to GNN layers and \cref{thm:sbgrprop} that enables \tool to locally recover graph structures. These theoretical developments allow for the efficient removal of proposal elements from $\mathcal{T}_l$, which fail the span check and hence cannot be a subgraph of the input, as we detail in \cref{sec:technical_filtering}.
Finally, in \cref{sec:technical_dfs} we demonstrate our graph building, which recovers the entire graph from the filtered set of possible subgraphs $\mathcal{T}_B^*$ using DFS. 

\subsection{Graph Gluing} 
\label{sec:gluing}
In this section we describe the process of gluing a $l$-hop building block $\gB = (V^B, E^B, c^B)$ to a graph $\gG = (V, E)$ at a vertex $c \in V$. The resulting set of graphs $\mathbb{G}$ contains all possible ways of attaching the non-overlapping parts of $\gB$ to $\gG$ at $c$, as shown in \cref{fig:glue}. To this end, we combine the 2 graphs by correctly matching equivalent nodes between them based on their features.
In particular, we return the empty set if the center of the building block $v'$ and the chosen node $v$ from the graph have different features. The same holds if the $l$-hop neighborhood of the chosen node does not match a subgraph of the building block. In all other cases we return the set of all possible graphs resulting from the gluing $\mathbb{G} = \glue{\gG, \gB, v}$. We describe how to efficiently perform gluing in \cref{app:gluing}.

\begin{wrapfigure}{R}{0.36\textwidth}
	\begin{minipage}{.36\textwidth}
        \begin{figure}[H]
            \vspace{-13mm}
            \includegraphics[width=\textwidth]{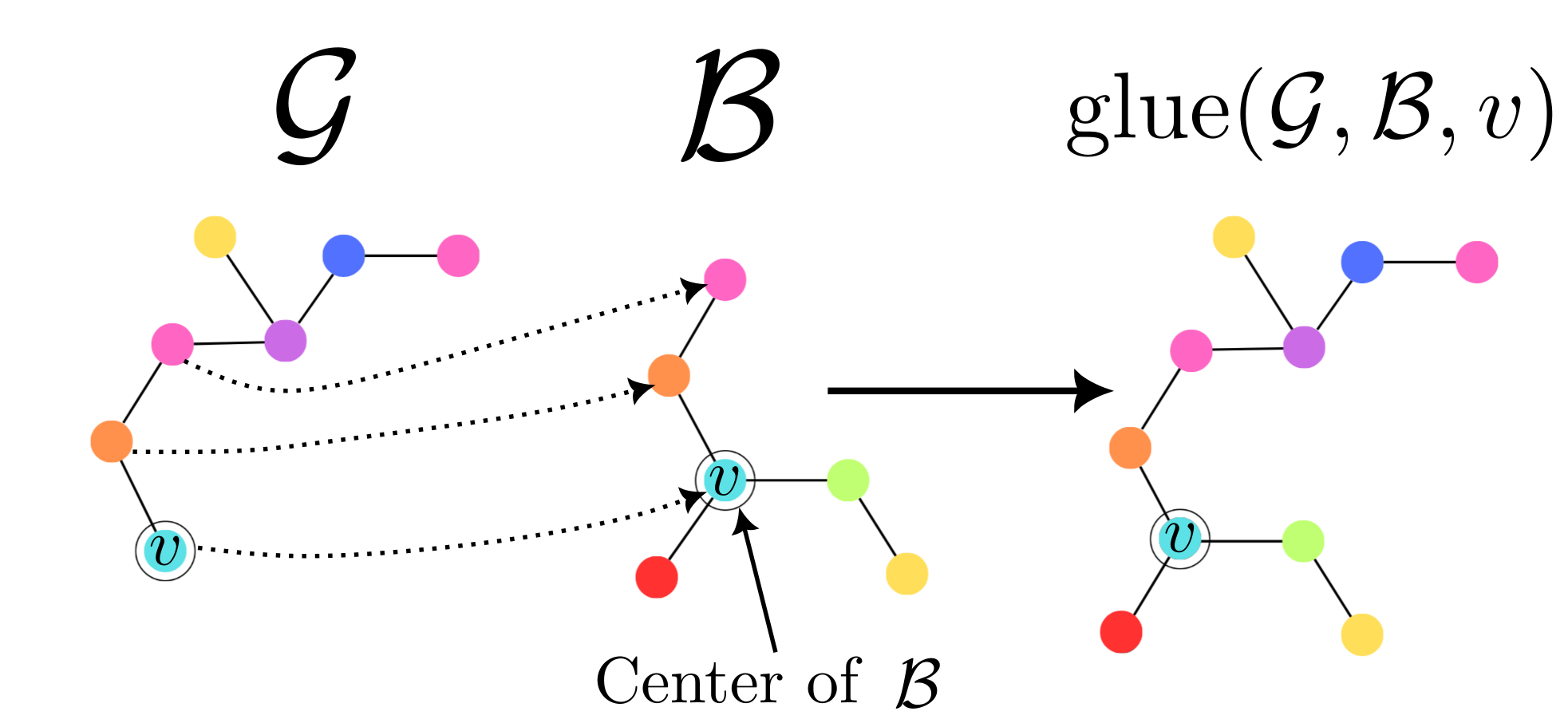}
            \vspace{-6mm}
            \caption{Glueing visualization}	
            \label{fig:glue}
        \end{figure}
        \vspace{-15mm}
    \end{minipage}
\end{wrapfigure}

\subsection{Theoretical Foundations of the Spancheck Filtering}
\label{sec:spancheck}
\paragraph{Span check for GNN layers} We now state our main result extending \cref{thm:spancheck_dager} to GNN layers. The proof is in \cref{app:proofs}.

\begin{restatable}{theorem}{spcheck}
	\label{thm:spancheck_ours}
	If $n < d$, $\mX^l_i \in \ColSpan(\grad{\mW^l})$ if and only if $\grad{\mY^l_i} \notin \RowSpan(\hat{\grad{\mY^l_i}})$, where $\hat{\grad{\mY^l_i}}$ denotes the matrix $\grad{\mY^l}$ with its $i$-th row removed.
\end{restatable}

This important generalization of the theory presented in DAGER provides an exact condition for recovering individual input vectors $\mX^l_i$ under any output gradient $\grad{\mY^l}$. In particular, when $\grad{\mY^l}$ is full-rank, the theorem recovers the statement of \cref{thm:spancheck_dager}. Otherwise, intuitively, the theorem states that we lose recoverability for inputs $\mX^l_i$ for which the corresponding row in $\grad{\mY^l}$ is linearly dependent of the rest of the rows in $\grad{\mY^l}$. The empricial experiments, illustrated in \cref{fig:gat_ablation} showcase that $\grad{\mY^l}$ is almost certainly full-rank for GATs, enabling the span check to accurately filter the entire input, making \cref{thm:spancheck_dager} still applicable most of the time. However, in general $\grad{\mY^l}$ can exhibit low-rankness, which limits the recovery of the entire input. In particular, \cref{fig:a_full_rankness} shows that for GCNs for small graphs the full-rankness assumption is violated. To this end, we present the following corollary (proven in \cref{app:proofs}) which outlines the implications of \cref{thm:spancheck_ours} specifically for GCNs:

\begin{restatable}{corollary}{fullrank}
	\label{clr:spancheck_full_rank}
	For $\grad{\mZ^l}$ of full-rank, $n < d$, if the (possibly normalized) adjacency matrix at layer $l$, $\mA \in \mathbb{R}^{n\times n}$, $\mX^l_i \in \ColSpan(\grad{\mW^l})$ if and only if $\mA^T_i \notin \ColSpan(\hat{\mA_i})$. Further, if $\mA$ is full-rank, then $\mX^l_i \in \ColSpan(\grad{\mW^l})$ for all $i=1,2,\dots, n$.
\end{restatable}
i.e. when $\mA$ is full-rank, all feature vectors at layer $l$ are recoverable via the spancheck. If this isn't the case, according to the main theorem we still recover most, but not all. In \cref{fig:a_full_rankness} and \cref{app:ablation} we demonstrate that even when $\mA$ is significantly rank-deficient, the majority of $\mX^l$ can still be recovered under the GCN setting. Our experiments on real-world datasets show that even in these cases \tool is able to partially recover the graph. Next we explain how we propagate our proposed building blocks through the GNN layers in order to be able to apply \cref{thm:spancheck_ours} to filter them.

\begin{figure*}
	\vspace{-5mm}
    \centering
	\begin{subfigure}{0.233\textwidth}
		\centering
		\includegraphics[width=.98\linewidth]{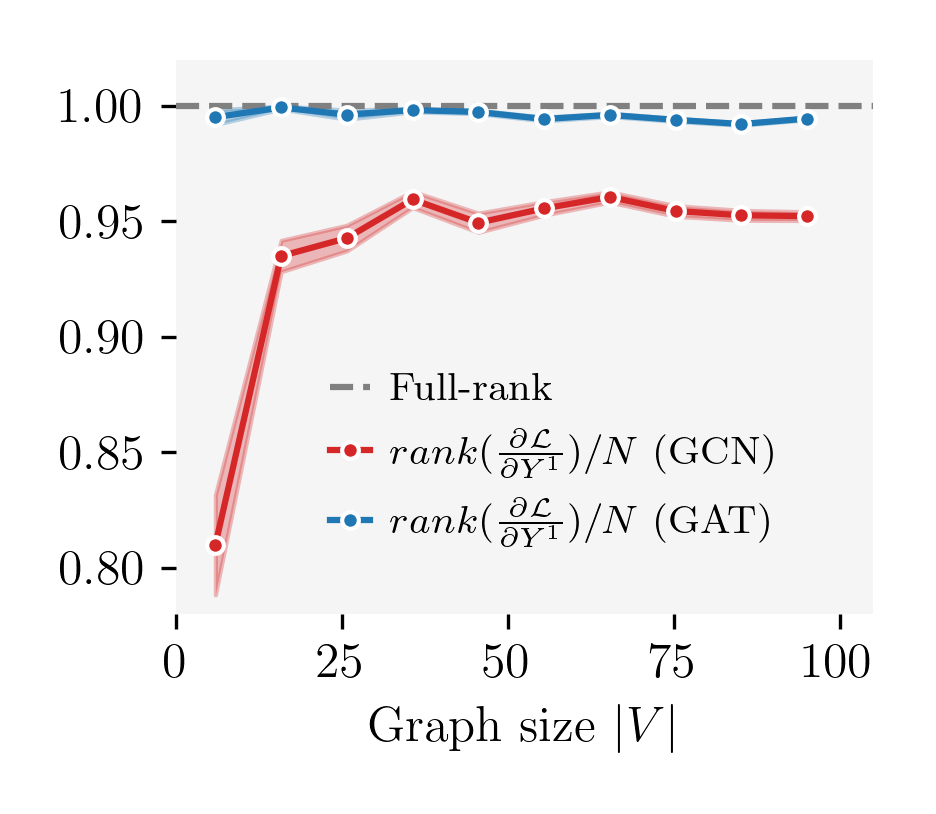}
		\caption{GCN vs GAT reconstructability (on Citeseer)}
		\label{fig:gat_ablation}
	\end{subfigure}
	\hspace{4mm}
	\begin {subfigure}{0.715\textwidth}
	\centering
		\includegraphics[width=.3\linewidth]{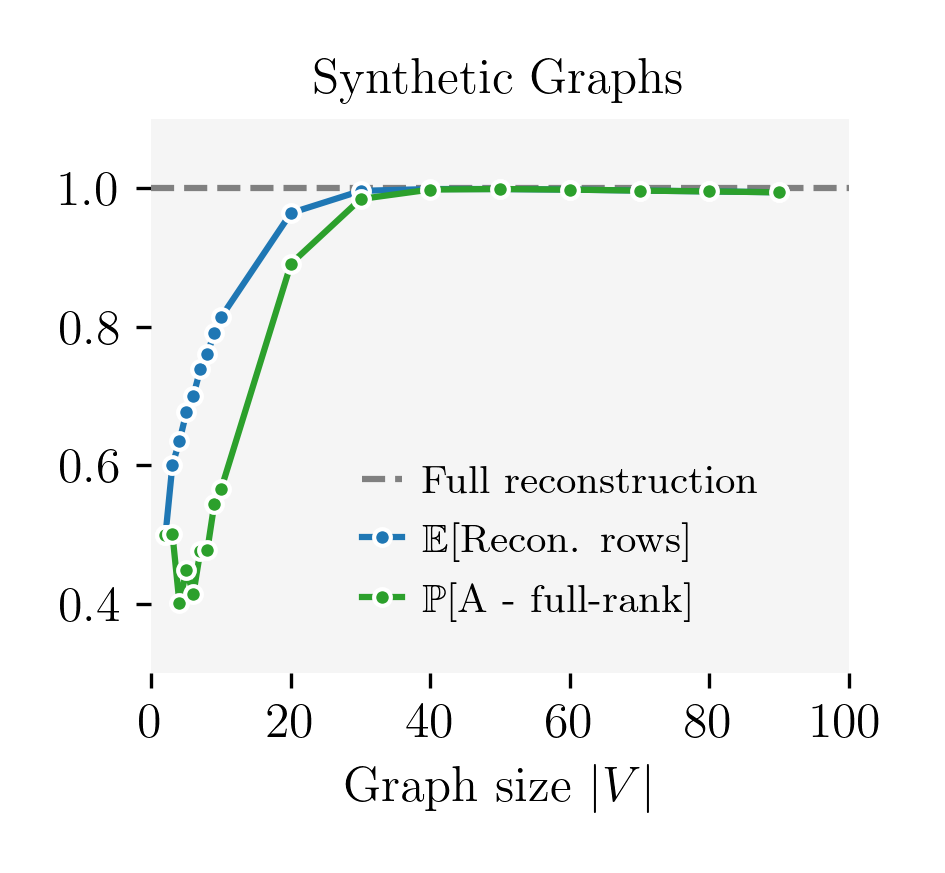}
		\includegraphics[width=.3\linewidth]{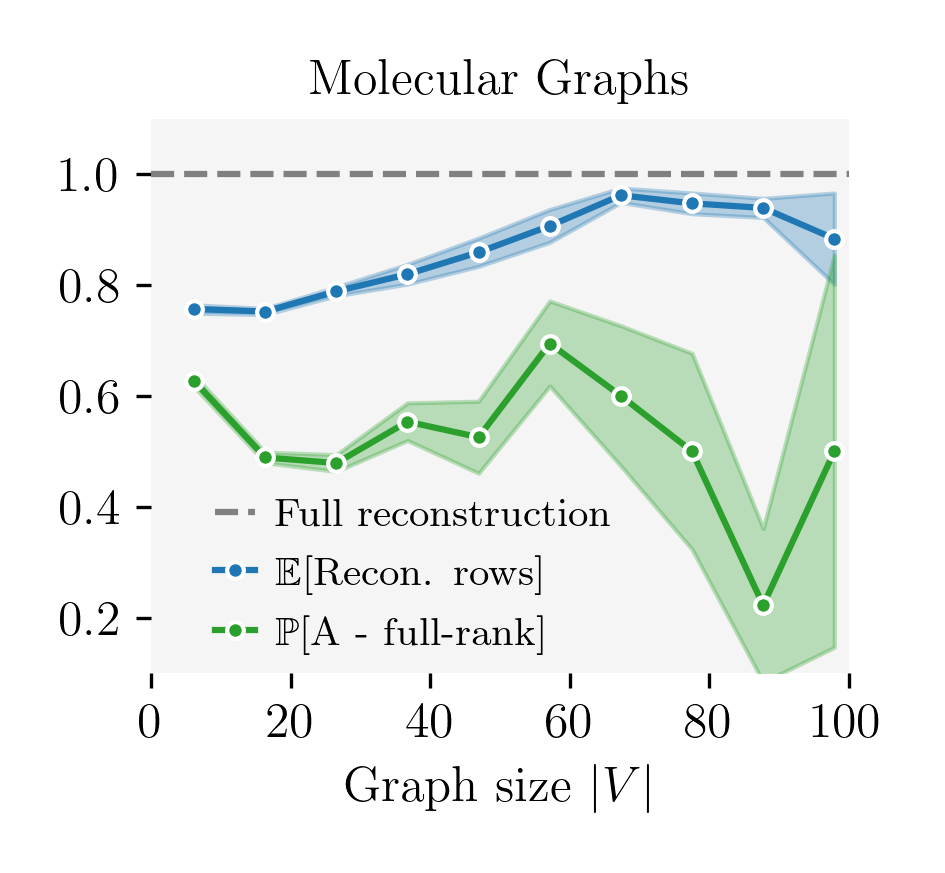}
		\includegraphics[width=.3\linewidth]{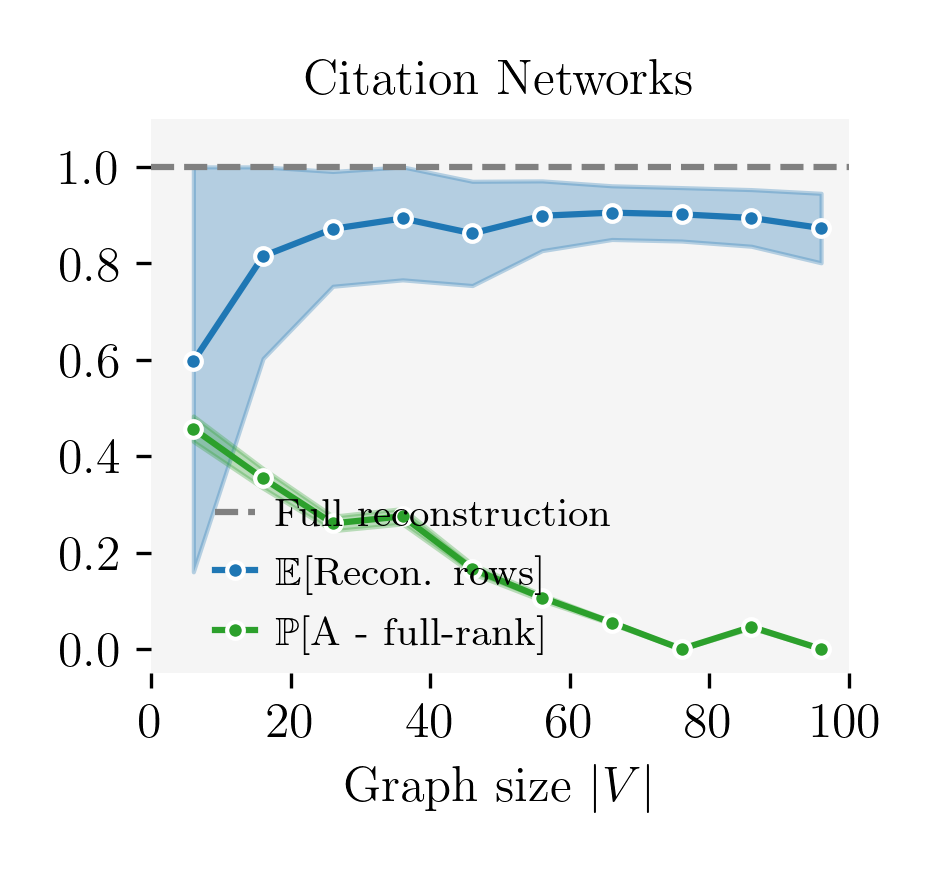}
		\caption{Impact of low-rankness of the adjacency matrix on reconstructability for GCNs for synthetic data (left), molecules (middle), and citation networks (right)}
		\label{fig:a_full_rankness}
	\end{subfigure}
	\vspace{-1mm}
	\caption{Ablation studies on how the data and architecture affect reconstructability}
	\vspace{-3mm}
\end{figure*}

\paragraph{Building block propagation}
In order to be able to apply \cref{thm:spancheck_ours} to filter a generated $l$-hop building block at the $l$-th layer we need to reconstruct its corresponding input embedding vector $\mX^l_j$ to the $l$-th layer. We now state our key result allowing us to recover this input vector.
\begin{restatable}{theorem}{sbgrprop}
	\label{thm:sbgrprop}
	For GNNs satisfying \cref{assmp:dependence}, propagating a correctly reconstructed building block $\gG^{l}_{v}$ centred at $v$ through the first $l$ GNN layers recovers the original embedding of $v$ at layer $l$:
	\[f_{l-1}(\gG^{l}_{v})[j] = X^{l}_i,\]
	where $i$ is the index of $v$ in the adjacency matrix of the original input graph $\gG$ and $j$ is the index of $v$ in the adjacency matrix of $\gG^l_v$.
\end{restatable}
Intuitively, this shows that a correctly recovered $l$-hop neighborhood is sufficient to compute the embeddings of $v$ at layer $l$, enabling the span check filter at this layer. The proof is available in \cref{app:proofs}, alongside a proof that GCNs and GATs satisfy \cref{assmp:dependence}. 

\subsection{Subgraph filtering}
\label{sec:technical_filtering}

Next, we begin our attack with the creation of the $L$-hop building blocks and reduce the search space via the span check mechanism. Even though \tool is applicable to any type of discrete features, we describe our methodology for one-hot encoded ones for notational convenience.

\paragraph{Recovering the single node feature vectors in $\mathcal{T}_0^*$}
We now describe how to recover the single node client feature vectors, that is the 0-hop neighborhoods in $\mathcal{T}_0^*$. We note that the direct enumeration approach by DAGER is often infeasible for graph data. 
To this end, instead of performing span checks on the entire input vectors, we do partial span checks on partial vectors, iteratively adding dimensions in order to control the complexity. Specifically, we first take the dimensions corresponding to a subset of features and filter those. To further constrain the search space, we then keep adding dimensions corresponding other features, followed by another filtering. 
We note the soundness of this partial filtering is covered by \cref{thm:spancheck_ours}, as performing the spancheck on a subspace is less restrictive compared to the full span check. 
Pseudocode and further explanations are provided \cref{app:fbyf}.

\paragraph{Creating 1-hop building blocks $\mathcal{T}_1^*$}
\begin{wrapfigure}[10]{r}{0.43\textwidth}
	\begin{minipage}{0.43\textwidth}
		\vspace{-8mm}
		\begin{algorithm}[H]
			\caption{Filtering using the spancheck}
			\label{alg:spancheck}
			\begin{algorithmic}[1]
				\Function{Filter}{$\mathcal{T}_{l}$, $\grad{\mW^l}$, $\tau$, $f_{l-1}$} 
				\State $\mathcal{T}^{*}_{l}\gets\{\}$
				\For{$\gG$ in $\mathcal{T}_{l}$}
				\State $v\gets\text{center}(\gG)$
				\If{$d(f_{l-1}(\gG)_v, \grad{\mW^l}) < \tau$}
				\State $\mathcal{T}^{*}_{l}\gets \mathcal{T}^{*}_{l} \cup \{\gG\}$
				\EndIf
				\EndFor
				\State \Return $\mathcal{T}^{*}_{L}$ 
				\EndFunction
			\end{algorithmic}
		\end{algorithm}
		\end{minipage}
\end{wrapfigure}
We first define the extension $\ext{v}$ of a node $v$ to be the set of all 1-hop building blocks that can be constructed by attaching exactly $\deg(v)$ nodes from $\mathcal{T}_0^*$ to $v$.
We note that we do not attach nodes $w \in \mathcal{T}_0^*$ to $v$ if the feature degree of $w$ is $0$, that is $\deg(w) = 0$. 
The set of all possible 1-hop building blocks $\mathcal{T}_1 = \bigcup_{v \in \mathcal{T}_0^*} \ext{v}$ is then defined as the set of all possible 1-hop neighborhoods that can be constructed by extending node from $\mathcal{T}_0^*$.
\cref{thm:sbgrprop} then allows us to exactly recover the first layer embedding for the center of each neighborhood. We are hence able to filter $\mathcal{T}_1$ by applying \cref{thm:spancheck_ours} on $\grad{\mW^1}$ to achieve the reduced set of 1-hop building blocks $\mathcal{T}_1^*$, as shown in \cref{alg:spancheck}.

\begin{wrapfigure}[21]{r}{0.56\textwidth}
	\vspace{-9mm}
	\begin{minipage}{0.56\textwidth}
	\begin{algorithm}[H]
		\caption{Creating the degree-$L$ building blocks}
		\label{alg:createbbs}
		\begin{algorithmic}[1]
			\Function{GenerateBBs}{$\{\mathcal{F}_i\}_{i\in\{[1,f]\}},\grad{\mW},\tau, \vf$} 
			\State $\mathcal{T}^{*}_{0}\leftarrow$ \Call{FilterNodes}{$\{\mathcal{F}_i\}$, $\grad{\mW_0}$, $\tau$}
			\State $\mathcal{T}_{1}\leftarrow\{\}$
			\For{$v$ in $\mathcal{T}^{*}_{0}$}
				\State $\mathcal{T}_{1}\leftarrow\mathcal{T}_{1} \cup \ext{v,\mathcal{T}_{0}^*}$
			\EndFor
			\State $\mathcal{T}^{*}_{1}\leftarrow$\Call{Filter}{$\mathcal{T}_{1},\grad{\mW_{1}},\tau,f_0$}
			\For{$l \gets 1, \dots, L-1$}
				\State $\mathcal{T}_{l+1}\leftarrow\{\}$
				\For{$G$ in $\mathcal{T}_{l}^{*}$}
					\State $S\leftarrow \{G\}$ \label{alg:createbbs_iter}
					\For{$v$ in $\dang{G}$}
						\State $S'\leftarrow \{\}$
						\For{$\gG',\gG^B$ in $S\times\mathcal{T}_{1}^{*}$}
							\State $S' \leftarrow S'\:\cup\:$\Call{Glue}{$\gG$, $\gG^B$, $v$, $l$ }
						\EndFor
						\State $S\leftarrow S'$					
					\EndFor
					\State $\mathcal{T}_{l+1}\leftarrow\mathcal{T}_{l+1} \cup S$ 
					\label{alg:createbbs_iter_end}
				\EndFor
				\State $\mathcal{T}^{*}_{l+1}\leftarrow$ \Call{Filter}{$\mathcal{T}_{l+1},\grad{\mW_{l+1}},\tau,f_{l-1}$}
			\EndFor
			\State \Return $\mathcal{T}^{*}_{L}$ 
			\EndFunction
		\end{algorithmic}
	\end{algorithm}
\end{minipage}
\end{wrapfigure}
\paragraph{Creating $l$-hop building blocks $\mathcal{T}_l^*$}
For a building block $\gB \in \mathcal{T}_l^*$, we define the \emph{dangling nodes} $\dang{\gB}$ as the set of all nodes $v \in \gB$ such that $\deg(v)$ (the ground-truth degree of the node) is greater than the number of its neighbors.
We extend the $l$-hop building blocks $\gB \in \mathcal{T}_l^*$ by calculating all possible gluings of $1$-hop building blocks $\gB' \in \mathcal{T}_1^*$ to all dangling nodes of $\gB$.
This is shown in lines \ref{alg:createbbs_iter}--\ref{alg:createbbs_iter_end} of \cref{alg:createbbs}. 
The resulting set is then called $\mathcal{T}_{l+1}$, which is then filtered by applying \cref{thm:spancheck_ours} on $\grad{\mW^{l+1}}$ to achieve the reduced set of $l+1$-hop building blocks $\mathcal{T}_{l+1}^*$.
We repeat the process explained above until we reach the desired $L$-hop neighborhoods, with the final spancheck performed on the first linear layer of the commonly used readout classification head.

\paragraph{Additional structure-based filtering}

To further restrict the proposal set of building blocks, we perform a consistency check to rule out blocks that cannot be part of the ground truth graph.
Specifically, for every building block $\gB \in \mathcal{T}_L^*$ and for every dangling node in it $v \in \dang{\gB}$ we assert that there exists a building block in $\mathcal{T}_L^*$ that we can glue it at $v$ to $\gB$. If this is not the case, we know that either $\gB$ is the input graph (which we check by computing the gradients produced by $\gB$), or it cannot be part of the ground truth graph and remove it from $\mathcal{T}_L^*$. We denote the resulting set $\mathcal{T}_B^*$.

\newpage
\subsection{Graph Building}
\label{sec:technical_dfs}

\begin{wrapfigure}[15]{r}{0.44\textwidth}
	\vspace{-17mm}
	\begin{minipage}{0.44\textwidth}
		\begin{algorithm}[H]

			\caption{DFS reconstruction}
			\label{alg:dfs}
			\begin{algorithmic}[1]
				\Function{DoDFS}{$\mathcal{T}^{*}_{B}$,$\grad{\mW}$, 		$\gG_{\text{curr}}$, $\mathcal{C}$}
				\If{$|dang(\gG_{\text{curr}})| == 0$}{\label{check_grad}}
					\State \Return $\Delta_{\gG_{\text{curr}}}, \gG_{\text{curr}}$\label{sol_found}
				\EndIf
				\State
				\State $\gG_{\text{TOP}}, d_{\text{TOP}} \leftarrow  \emptyset,\infty$
				\State $\mathbb{B}_{\text{ord}}\leftarrow \text{Order}(\mathcal{T}^{*}_{B})$
				\State $v \leftarrow Sample(\{v \in dang(G_{\text{curr}})\})$\label{sample_db}
				\State  $\mathbb{G}_{\text{new}} \leftarrow \text{Branch}		(\mathbb{B}_{\text{ord}}$, $\gG_{\text{curr}}$, $v$)\label{branch}
				
				\For{$\gG$ in $\mathbb{G}_{\text{new}}$}{\label{11}}
					\State $d', \gG' \leftarrow $\Call{DoDFS}{$\mathcal{T}^{*}_{B}, \grad{\mW}, \gG, \gC$}
					\If{$d'==0$}{\label{13}}
						\State \Return $0, \gG'$\label{14}
					\ElsIf{$d_{\text{TOP}}>d'$}
					\State $d_{\text{TOP}}, \gG_{\text{TOP}}\leftarrow d',\gG'${\label{16}}
					\EndIf  
				\EndFor
				\State\Return $d_{\text{TOP}}, \gG_{\text{TOP}}$
				\EndFunction
			\end{algorithmic}
		\end{algorithm}

	\end{minipage}
\end{wrapfigure}

Finally, we describe how we leverage depth-first search to combine the filtered set of building blocks $\mathcal{T}_B^*$ into our final graph reconstruction.

Specifically, each node of our DFS exploration tree represents a partially reconstructed client graph $\gG_{\text{curr}}$, and at each branch we choose an arbitrary dangling node $v\in\gG_{\text{curr}}$ (Line \ref{sample_db} in \cref{alg:dfs}) and generate all possible graphs $\mathbb{G}=\glue{\gG_{\text{curr}}, \gB, v}$ that can be created by gluing a building block $\gB\in\mathcal{T}_B^*$ at $v$ (Line \ref{branch} in \cref{alg:dfs}). This is done by iterating over all the building blocks in $\mathcal{T}_B^*$, checking for each one if it is possible to glue it to $\gG_{\text{curr}}$ and if so, saving the extended graph as a new branch. The pseudocode for the branching is provided in \cref{app:branching}. During branching, we also take care of recovering possible cycles within the reconstructions by overlapping nodes in $\gG_{\text{curr}}$ that might coincide. This is also elaborated on in \cref{app:branching}. If $\gG_{\text{curr}}$ has no more dangling nodes, we calculate the distance between its and the client gradients:
\begin{equation}
	\label{eq:check_grad}
	\Delta_\gG = \min_{c\in \mathcal{C}}\|\tfrac{\partial\mathcal{L}(\gG, c)}{\partial\mW} - \grad{\mW}\|_F, 
\end{equation} where $\mathcal{C}$ is the set of all possible labels and $\|\cdot\|_F$ is the standard Frobenius norm. If $\Delta_{\gG_{\text{curr}}}=0$  the algorithm terminates early with correct graph (Line \ref{14} in \cref{alg:dfs}). Otherwise, we return the best possible reconstruction in gradient distance (Lines \ref{11}-\ref{16} in \cref{alg:dfs}).

\paragraph{Building block ordering}
A key factor in the performance of the algorithm is the order in which we visit nodes of the exploration tree. We use heuristic ordering based on a score $S(\gB)$ we define for every building block $\gB\in\mathcal{T}_B^*$. We first define the score $S_v(\gB)$ to be equal to the lowest span check distance $d(\gG, \grad{\mW^L})$ of a building block $\gB$ that can be glued to $\gG$ at $v$. The score for the entire block is then calculated as the sum of the vertex scores $S(\gB) = \sum_{v}S_v(\gB)$. Intuitively, this score represents how compatible any building block from $\mathcal{T}_B^*$ is with the other building blocks that passed the spanchecks. Again, intuitively, the building blocks which are part of the input will be more compatible with the other building blocks from the input than potential false positives.

\paragraph{Uniqueness heuristic}In some settings, such as social or citation networks we notice that the feature vectors of different nodes are almost surely unique. In these settings, we leverage a heuristic which during reconstruction always overlaps any two nodes with the same features. Further, in these settings we never try to glue the same building block twice, as $L$-hop neighborhoods are analogically unique. The heuristic drastically reduces the search space and the time for convergence of \tool.

\section{Evaluation}\label{sec:eval_top}

In this section we evaluate \tool's performance against prior gradient leakage methods. 

We begin by detailing our experimental setup and the baseline attacks considered, along with introducing a novel set of metrics, specifically designed to jointly assess the differences in node features and graph topology between the true client graphs and their reconstructions. The experimental results demonstrate \tool's substantial improvements over existing attacks, achieving superior reconstruction accuracy and versatility. Namely, we demonstrate that \tool remains effective across a wide range of architectural changes, GNN model types, data modalities, and task scenarios.

\subsection{Evaluation metric}
\label{sec:eval}
We found it necessary to design our own set of metrics, as prior graph-related similarity measurements were not suitable for evaluating gradient inversion attacks. A discussion on the desired qualities of the metrics, and the reasoning behind their design can be found in \cref{app:gsm_disc}.

To this end, we introduce the \emph{\textbf{G}raph \textbf{S}imilarity \textbf{M}etrics} (GSM) - a set of metrics designed to evaluate the similarity of a pair of graphs $\gG$ and $\hat{\gG}$ under the name $\text{\metric{N}}$, the details of which we showcase in \cref{app:gsm_disc}. A key highlight of GSM is that it assesses structures of varying globality, while ensuring the metric is invariant to graph isomorphism through rigorous node matching. 

We utilise 3 separate instances of the metric - namely for $N = 0,1,2$, where larger-hop neighborhoods are used to capture more structural information. It is important to note that all measurements are scaled by a factor of $\frac{\min(|\mathcal{V}|,|\hat{\mathcal{V}}|)}{\max(|\mathcal{V}|,|\hat{\mathcal{V}}|)}$ to penalize reconstructions of incorrect size. We further report the percentage of exactly reconstructed graphs, denoted by FULL in the result tables.

To measure the perceived reconstruction quality and compare it to our GSM set of metrics, as well as to confirm that our metrics are fair with respect to the baseline attacks, we further conducted a human evaluation study. In \cref{tab:study}, we show that our metrics are highly correlated with human perception. Further details about how this study was conducted are shown in \cref{app:experiments}.

\begin{wrapfigure}[4]{r}{0.55\textwidth}
	\vspace{-18mm}
	\begin{minipage}{0.55\textwidth}
		\begin{table}[H]
			\centering
			\caption{Comparison of GSM and human evaluation.}
			\vspace{-2mm}
			\label{tab:study}
			\begin{tabular}{lcccc}
				\toprule
				& GSM-0 & GSM-1 & GSM-2 & Human \\
				\midrule
				\tool & \textbf{72.6}    & \textbf{67.8}    & \textbf{66.9}    & \textbf{70.6}        \\
				DLG  & 24.2             & 10.5             & 12.0             & 6.5                  \\
				\bottomrule
			\end{tabular}
		\end{table}
	\end{minipage}
\end{wrapfigure}

\subsection{Experimental setup}
Next, we describe our experimental setup, including the architecture of the attacked models, the client datasets used, and the hardware required by the attacker.

\paragraph{Architecture details}
Unless otherwise specified, all of our attacks are applied on 2-layer GNNs ($L=2$) with a hidden embedding dimension $d'=300$ and a ReLU activation. For GAT experiments 2-headed attention was used (adapting \tool{} to GATs with more heads is analogous). All networks also feature a 2-layer feedforward network for performing the readout --- a common depth for GNNs \cite{kipf2016semi}. Given the depth restrictions, we recover building blocks up to layer 2, with the first readout layer being used for the relevant filtering of the largest blocks. In \cref{table:network_sizes_table}, we show that our attack is robust with respect to changes in these architectural parameters.

\paragraph{Evaluation datasets}
We evaluate on three different types of graph data -- chemical data, citation and social networks. For the chemical experiments, we evaluate on molecule property prediction data, where molecules are represented as graphs and each node is a given atom. 
We follow the common convention to omit hydrogen atoms in the graphs. 
Each node is embedded by concatenating the one-hot encodings of 8 features \citep{xu2018powerful, wu2020comprehensive}, namely the atom type, formal charge, number of bonds, chirality, number of bonded hydrogen atoms, atomic mass, aromaticity and hybridization \citep{rong2020self}. We evaluate \tool on 3 well-known chemical datasets -- Tox21, Clintox, and BBBP, introduced by the \textbf{MoleculeNet} benchmark \citep{wu2018moleculenet}.

For the citation networks experiments, we apply \tool on the CiteSeer\citep{CiteSeer} dataset, which features a single graph with 3312 nodes representing scientific publications classified into one of six classes. The edges between them represent citations, and the features of any node is a 0/1-valued word vector of length 3703 indicating the absence/presence of a keyword in the abstract.

Finally, for the social network experiments we use the Pokec\citep{pokec} dataset, containing 1.6 million nodes (users), where connections represent friendship between users. We chose discrete node features where every feature with frequency less that $1\%$ was categorized as "Other", as these entries usually contain irrelevant outliers. This resulted in 36 total discrete features.

To simulate a federated learning environment, in the latter two settings we sample subgraphs of a given size from the datasets for each of the FL clients and use cluster classification objective, where each subgraph is classified as the most common class among the comprising nodes. We chose the sampling distribution: 20 graphs with 1-10 nodes, 40 graphs with 10-20 nodes, 30 graphs with 20-30 nodes, and 10 graphs with 30-40 nodes, closely mirroring the distribution of the Tox21 dataset. 
\paragraph{Computational requirements}
We provide an efficient GPU imlementation, where each experiment has been run on a NVIDIA L4 Tensor Core GPU with less than 40GB of CPU memory.

\subsection{Baseline Attacks}
We adapt the DLG attack \citep{zhu2019deep}, a standard continuous attack, and TabLeak, an attack purposefully designed for recovering discrete tabular data. As described in \citep{vero2023tableak}, all input features are first passed through an initial sigmoid layer to ensure they are in the interval (0, 1). Similarly, we ensure the adjacency matrix $\mA$ is symmetric by optimizing over the upper triangle, and apply a softmax operation over the dummy labels to convert them to probabilities. Finally, we generate a prediction graph by connecting all nodes $v_i, v_j$ corresponding to $\sigmoid{(\mA)}_{ij} \geq 0.5$. Additionally, we test both baselines when they are given the correct adjacency matrix $\mA$. In all cases we provide the attack with the correct number of nodes to ensure that $\mX$ and $\mA$ have the correct shape. We demonstrate that, even when the baselines have a significant amount of prior knowledge, \tool significantly outperforms them (see \cref{fig:example} and \cref{table:main_vs_baseline}).

\subsection{Experimental results}
Next, we evaluate the baselines and \tool and show that \tool outperforms the existing baselines across all defined metrics. Further, \tool is applicable across a variety of datasets and settings, including being depth- and width-agnostic, and far more scalable and robust. In all measurements we quote the mean value of the metric, as well as the 95\% confidence interval around it, measured by generating 10,000 random sample sets via bootstrapping.

\begin{table*}\centering
	\vspace{-4mm}
	\caption{Results (in \%) of experiments on the 3 dataset types -- Tox21 (chemical), CiteSeer (citation), Pokec (social). Here "$+A$" refers to the baseline attack with the input adjacency matrix given.} \label{table:main_vs_baseline}
	\vspace{-2mm}
	\renewcommand{\arraystretch}{1.2}
	
	\newcommand{\temp}[1]{\textcolor{red}{#1}}
	\newcommand{\noopcite}[1]{} 
	
	\newcommand{\skiplen}{0.002\linewidth} 
	\newcommand{\rlen}{0.01\linewidth} 
	\newcolumntype{R}{>{$}r<{$}}
	\resizebox{\linewidth}{!}{
		\begin{tabular}{@{}c l RRRRR p{\skiplen} RRRRR @{}} \toprule
			&& \multicolumn{5}{c}{\text{GCN}} && \multicolumn{5}{c}{\text{GAT}}\\
			\cmidrule(l{5pt}r{5pt}){3-7} \cmidrule(l{5pt}r{5pt}){9-13} 
			&& \multicolumn{1}{c}{\text{\metric{0}}} & \multicolumn{1}{c}{\text{\metric{1}}}  & \multicolumn{1}{c}{\text{\metric{2}}} & \multicolumn{1}{c}{\text{FULL}} & \multicolumn{1}{c}{\text{Time [h]}}
			&& \multicolumn{1}{c}{\text{\metric{0}}} & \multicolumn{1}{c}{\text{\metric{1}}}  & \multicolumn{1}{c}{\text{\metric{2}}} & \multicolumn{1}{c}{\text{FULL}} & \multicolumn{1}{c}{\text{Time [h]}} \\ \midrule
			\multirow{5}{*}{Tox21}

			& \tool & \mathbf{86.9^{+4.2}_{-5.7}} & \mathbf{83.9^{+5.2}_{-6.9}} & \mathbf{82.6^{+5.7}_{-7.4}} &  \mathbf{68.0\pm1.7} & 14.3 && 92.9^{+3.8}_{-5.8} & \mathbf{90.7^{+5.0}_{-7.1}} & \mathbf{89.9^{+5.8}_{-7.2}} & \mathbf{75.0 \pm 1.8} & 10.8\\ 
			& DLG & 31.8^{+4.5}_{-4.3} & 20.3^{+5.5}_{-4.8} & 22.8^{+6.6}_{-5.6} &  1.0\pm0.2 & 3.3 && 96.0\pm 0.32 & 9.3^{+4.4}_{-4.9} & 6.5^{+3.9}_{-4.1} & 2.0 \pm 0.3 & \mathbf{4.2} \\
			& DLG $+A$ & 54.7^{+3.9}_{-4.2} & 60.1^{+4.6}_{-5.2} & 76.7^{+3.6}_{-4.8} &  1.0\pm0.2 & \mathbf{3.1} && \mathbf{96.5 \pm 0.34} & 69.7^{+4.1}_{-4.2} & 81.3^{+3.4}_{-3.6} & 2.0 \pm 0.3 & 4.5 \\
			& TabLeak & 25.1^{+5.1}_{-4.3} & 12.4^{+5.5}_{-4.3} & 10.8^{+5.6}_{-3.9} &  1.0\pm0.2 & 13.1 && 73.7^{+2.6}_{-2.0} & 7.2^{+5.2}_{-4.9} & 10.0\pm4.8 & 1.0 \pm 0.2 & 6.0
			\\
			& TabLeak $+A$ & 55.6^{+3.9}_{-3.9} & 57.7^{+4.1}_{-4.6} & 73.8^{+2.8}_{-3.5} & 1.0\pm0.2 & 12.3 && 75.1^{+2.5}_{-1.9} & 74.9^{+2.1}_{-1.9} & 84.2^{+1.5}_{-1.3} & 1.0 \pm 0.2 & 6.0
			\\
			
			\midrule
			\multirow{5}{*}{CiteSeer}
			&\tool &  62.5^{+7.7}_{-8.2} & \mathbf{31.0^{+8.0}_{-7.8}}
			 & \mathbf{31.6^{+8.1}_{-8.1}} & \mathbf{20.0\pm{0.8}} & \mathbf{2.5}  && \mathbf{79.3^{+4.7}_{-6.3}} & \mathbf{69.1^{+6.1}_{-6.4}} & \mathbf{69.6^{+6.2}_{-6.0}} & \mathbf{61.0\pm{1.6}} & \mathbf{2.1}
			\\
			& DLG & \mathbf{65.7^{+1.6}_{-1.7}} & 0.0^{+0.0}_{-0.0} & 0.1^{+0.1}_{-0.1} & 0.0\pm{0.0} & 25.6 && 65.7^{+1.6}_{-1.6} & 0.0^{+0.0}_{-0.0} & 0.0^{+0.0}_{-0.0} & 0.0\pm{0.0} & 26.3
			\\
			& DLG $+A$ & \mathbf{65.7^{+1.6}_{-1.6}} & 0.0^{+0.0}_{-0.0} & 0.0^{+0.0}_{-0.0} & 0.0\pm{0.0} & 26.8 && 65.7^{+1.7}_{-1.6} & 0.0^{+0.0}_{-0.0} & 0.0^{+0.0}_{-0.0} & 0.0\pm{0.0} & 28.2\\
			& TabLeak & 65.2^{+2.5}_{-2.4} & 0.0^{+0.0}_{-0.0} & 0.3^{+0.4}_{-0.3} & 0.0\pm{0.0} & 172.7 && 65.0^{+2.4}_{-2.3} & 0.0^{+0.0}_{-0.0} & 0.0^{+0.0}_{-0.0} & 0.0\pm{0.0} & 177.4\\
			& TabLeak $+A$ & 65.3^{+2.4}_{-2.2} & 0.0^{+0.0}_{-0.0} & 0.0^{+0.0}_{-0.0} & 0.0\pm{0.0} & 172.9 && 65.1^{+2.4}_{-2.3} & 0.0^{+0.0}_{-0.0} & 0.0^{+0.0}_{-0.0} & 0.0\pm{0.0} & 171.5\\
			
			\midrule
			\multirow{5}{*}{Pokec}
			&\tool&\mathbf{58.3^{+5.9}_{-5.9}} & \mathbf{30.7^{+7.9}_{-7.8}} & \mathbf{35.8^{+8.3}_{-7.9}} & \mathbf{15.0\pm{0.8}} & \mathbf{0.3} && \mathbf{97.2^{+1.6}_{-1.9}} & \mathbf{93.5^{+3.4}_{-4.2}} & \mathbf{96.3^{+1.9}_{-2.3}} & \mathbf{79.0\pm{1.8}} & \mathbf{0.5}\\
			& DLG&38.1^{+1.2}_{-1.2} & 0.0^{+0.0}_{-0.0} & 7.9^{+1.8}_{-1.8} & 0.0\pm{0.0} & 0.6&&35.8^{+1.2}_{-1.1} & 0.0^{+0.0}_{-0.0} & 0.1^{+0.1}_{-0.1} & 0.0\pm{0.0} & 1.2\\
			& DLG $+A$&37.4^{+1.2}_{-1.2} & 0.2^{+0.3}_{-0.2} & 26.3^{+2.1}_{-2.2} & 0.0\pm{0.0} & 0.5&&37.7^{+1.1}_{-1.1} & 0.1^{+0.1}_{-0.1} & 16.9^{+1.9}_{-1.9} & 0.0\pm{0.0} & 1.0\\
			& TabLeak&37.1^{+1.4}_{-1.3} & 0.1^{+0.1}_{-0.1} & 2.9^{+1.5}_{-1.3} & 0.0\pm{0.0} & 6.6&&37.0^{+1.1}_{-1.1} & 0.0^{+0.0}_{-0.0} & 1.1^{+0.9}_{-0.7} & 0.0\pm{0.0} & 12.3\\
			& TabLeak $+A$&37.9^{+1.3}_{-1.3} & 0.5^{+0.5}_{-0.4} & 20.1^{+2.8}_{-2.9} & 0.0\pm{0.0} & 5.7&&37.7^{+1.4}_{-1.4} & 0.3^{+0.3}_{-0.2} & 23.4^{+1.8}_{-1.8} & 0.0\pm{0.0} & 11.4\\
			
			\bottomrule
		\end{tabular}
	}
\end{table*}

\begin{figure*}
	\centering
	\vspace{-4.5mm}
	\includegraphics[width=0.9\linewidth]{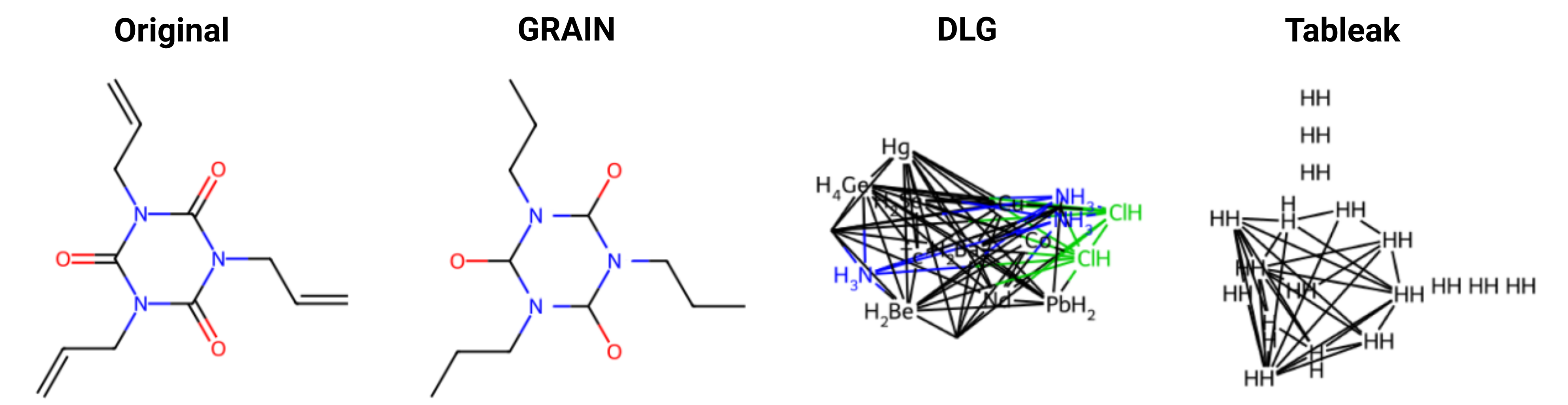}
	\vspace{-3.5mm}
	\caption{Examples molecule reconstructions. Multivalent interactions are not recovered, as they are not considered by the GNN.}
	\label{fig:example}
	\vspace{-4mm}
\end{figure*}
\paragraph{Main experiments} We first apply the algorithms DLG, TabLeak and \tool to the 3 types of dataset on both the GCN and GAT architectures. 
We observe in \cref{table:main_vs_baseline} that \tool achieves a much higher partial reconstruction rate (up to 96\%) compared to any baseline.
This remains true even when the baseline is informed about the input adjacency matrix $\mA$. Without $\mA$, baseline performance notably drops with neighborhood size, showing the baselines' inability to recover the structure.
Beyond partial reconstruction, \tool is further able to recover up to 80\% of the dataset exactly, while the baselines achieve this only in the case of very small graphs. Further, \tool closely matches TabLeak's runtime and is much faster than all baselines with the node uniqueness heuristic.
For visual inspection, we also include a comparison of a fully reconstructed molecule in \cref{fig:example}, with further examples in \cref{app:experiments}. 

As shown, \tool can also be effectively applied to both architectures, consistently achieving a higher score on the GAT architecture. This increase happens because as we showed in \cref{fig:gat_ablation}, $\grad{\mY^l}$ is almost certainly full-rank, where \cref{clr:spancheck_full_rank} enables the direct recovery of all inputs.

For the Tox21 baseline experiments we used a LBFGS optimizer for more stable and higher quality results. However, for CiteSeer and Pokec, where the client input space is much larger, we instead ran SGD due to time constraints. In \cref{app:ablation}, we show results with LBFGS on 10 times fewer samples.

\begin{table*}
	\vspace{-4mm}
	\caption{Comparison of \tool and baselines on chemical datasets for GCNs on graphs of size $n$.} 
	\label{table:comparison_nn}

	\vspace{-3mm}
	\newcommand{\twocol}[1]{\multicolumn{2}{c}{#1}}
	\newcommand{\threecol}[1]{\multicolumn{3}{c}{#1}}
	\newcommand{\fivecol}[1]{\multicolumn{5}{c}{#1}}
	\newcommand{\ninecol}[1]{\multicolumn{9}{c}{#1}}
	
	\newcommand{\bsz}{Batch Size~}
	\newcommand{\certified}{{CR(\%)}}
	
	\renewcommand{\arraystretch}{1.2}
	
	\newcommand{\ccellt}[2]{\colorbox{#1}{\makebox(20,8){{#2}}}}
	\newcommand{\ccellc}[2]{\colorbox{#1}{\makebox(8,8){{#2}}}}
	\newcommand{\ccells}[2]{\colorbox{#1}{\makebox(55,8){{#2}}}}
	
	\newcommand{\temp}[1]{\textcolor{red}{#1}}
	\newcommand{\noopcite}[1]{} 
	
	\newcommand{\skiplen}{0.01\linewidth} 
	\newcommand{\rlen}{0.01\linewidth} 
	\newcolumntype{R}{>{$}r<{$}}
	\resizebox{\linewidth}{!}{
		\begin{tabular}{@{}l RRR p{\skiplen}  RRR p{\skiplen} RRR @{}} \toprule
			& \threecol{$n\leq15$} && \threecol{$16\leq n\leq 25$} && \threecol{$26\leq n$}\\
			
			\cmidrule(l{5pt}r{5pt}){2-4} \cmidrule(l{5pt}r{5pt}){6-8} \cmidrule(l{5pt}r{5pt}){10-12} 
			
			& \multicolumn{1}{c}{\text{\metric{0}}} & \multicolumn{1}{c}{\text{\metric{2}}} & \multicolumn{1}{c}{\text{FULL}}
			&& \multicolumn{1}{c}{\text{\metric{0}}} & \multicolumn{1}{c}{\text{\metric{2}}} & \multicolumn{1}{c}{\text{FULL}}
			&& \multicolumn{1}{c}{\text{\metric{0}}} & \multicolumn{1}{c}{\text{\metric{2}}} & \multicolumn{1}{c}{\text{FULL}}\\ 
			\midrule			
			\multirow{1}{*}{\tool} 
			& \mathbf{93.0^{+3.4}_{-5.4}} & \mathbf{91.6^{+3.8}_{-6.3}} &  \mathbf{81.9\pm1.7} &&
			\mathbf{81.7^{+3.9}_{-4.8}} & \mathbf{74.8^{+5.8}_{-6.3}} &  \mathbf{43.6\pm1.1} &&
			50.1^{+6.8}_{-7.1} & 39.2^{+8.5}_{-7.7} &  \mathbf{5.1\pm0.6} \\

			\multirow{1}{*}{DLG } 
			& 27.4^{+4.2}_{-3.8} & 13.3^{+5.7}_{-4.6} &  1.0\pm0.2 &&
			25.5^{+3.9}_{-3.5} & 16.7^{+5.2}_{-4.4} &  0.9\pm0.2 &&
			25.4^{+4.8}_{-4.3} & 14.8^{+6.4}_{-5.3} &  0.0\pm0.0 \\

			\multirow{1}{*}{DLG $+A$} 
			& 52.1^{+3.1}_{-3.3} & 71.3^{+3.1}_{-3.9} &  1.0\pm0.2 && 53.7^{+3.2}_{-3.5} & 75.3^{+3.0}_{-3.7} &  0.9\pm0.2 && 53.0^{+4.3}_{-4.8} & 72.6^{+4.1}_{-5.8} &  0.0\pm0.0 \\

			\multirow{1}{*}{TabLeak} 
			& 30.3^{+5.0}_{-4.4} & 15.4^{+5.8}_{-4.8} &  1.9\pm0.3 &&
			15.7^{+3.0}_{-2.2} & 2.1^{+2.2}_{-1.1} &  0.0\pm0.0 &&
			13.0^{+3.3}_{-2.3} & 2.8^{+4.1}_{-1.9} &  0.0\pm0.0  \\
						
			\multirow{1}{*}{TabLeak $+A$} 
			& 53.9^{+4.0}_{-4.2} & 72.9^{+3.4}_{-3.9} &  1.9\pm0.3 &&
			57.1^{+3.1}_{-3.5} & 71.4^{+2.9}_{-3.6} &  0.0\pm0.0 &&
			\mathbf{56.1^{+2.9}_{-3.3}} & \mathbf{74.4^{+2.3}_{-3.1}} &  0.0\pm0.0\\
			
			\bottomrule
		\end{tabular}
	}
\end{table*}
\paragraph{Effect of graph size on reconstruction}
In \cref{table:comparison_nn} and \cref{tab:pokec} we show how \tool performs on graphs of different sizes under the same setting as our main experiments (for GCNs and GATs respectively). In the chemical setting (\cref{table:comparison_nn}) molecules are divided into groups of $n\leq15$, $16\leq n\leq 25$ or $n\geq26$ nodes, aggregated across the 3 chemical datasets. We notice that \tool significantly outperforms the baselines for smaller graphs, but the performance decreases on the largest groups in both the chemical and the social networks setting due to timeouts (15 minutes) during graph building. 
That said, the social network setting (\cref{tab:pokec}), allows us to apply the node uniqueness heuristic, allowing us to scale to much larger graphs of size up to $n=60$ nodes.
Further, our work still manages to reconstruct a fraction of the large graphs exactly, which is impossible for the baseline models,  even in the more difficult chemical setting (discussed in \cref{sec:limitations}).

\begin{table*}
	\vspace{-2mm}
	\centering
	\begin{minipage}[t]{0.505\textwidth}
		\centering
		\caption{Reconstruction results in \% for \tool on Citeseer and GATs for different ablations.}
		\label{tab:gat_additional}
		\newcommand{\twocol}[1]{\multicolumn{2}{c}{#1}}
		\newcommand{\threecol}[1]{\multicolumn{3}{c}{#1}}
		\newcommand{\fivecol}[1]{\multicolumn{5}{c}{#1}}
		\newcommand{\ninecol}[1]{\multicolumn{9}{c}{#1}}
		
		\newcommand{\bsz}{Batch Size~}
		\newcommand{\certified}{{CR(\%)}}
		
		\renewcommand{\arraystretch}{1.2}
		
		\newcommand{\ccellt}[2]{\colorbox{#1}{\makebox(20,8){{#2}}}}
		\newcommand{\ccellc}[2]{\colorbox{#1}{\makebox(8,8){{#2}}}}
		\newcommand{\ccells}[2]{\colorbox{#1}{\makebox(55,8){{#2}}}}
		
		\newcommand{\temp}[1]{\textcolor{red}{#1}}
		\newcommand{\noopcite}[1]{} 
		
		\newcommand{\skiplen}{0.01\linewidth} 
		\newcommand{\rlen}{0.01\linewidth} 
		\newcolumntype{R}{>{$}r<{$}}
		\resizebox{\linewidth}{!}{
		\begin{tabular}{@{}l RRRR @{}} \toprule

				& \multicolumn{1}{c}{\text{\metric{0}}} & \multicolumn{1}{c}{\text{\metric{1}}} & \multicolumn{1}{c}{\text{\metric{2}}} &  \multicolumn{1}{c}{\text{FULL}} \\ 
				
				\midrule
				
				\multirow{1}{*}{Default} 
				& \mathbf{79.3^{+4.7}_{-6.3}} & \mathbf{69.1^{+6.1}_{-6.4}} & \mathbf{69.6^{+6.2}_{-6.0}} & \mathbf{61.0\pm1.6}
				\\

				\multirow{1}{*}{{No degree}}
				& 59.7^{+6.8}_{-7.2} & 42.7^{+6.3}_{-6.6} & 43.2^{+6.4}_{-6.6} & 32.0\pm1.1\\

				\multirow{1}{*}{{No heuristic}}
				& 64.6^{+3.5}_{-4.2} & 52.1^{+4.7}_{-5.3} & 52.4^{+4.6}_{-5.2} & 44.0\pm 1.3\\
				
				\bottomrule
		\end{tabular}}
	\end{minipage}
	\hfill
	\begin{minipage}[t]{0.47\textwidth}
		\centering
		\caption{Reconstruction results in \% for \tool on Tox21 for GCNs in various settings}
		\label{table:extra_table}
		\vspace{-1.9mm}
		\newcommand{\twocol}[1]{\multicolumn{2}{c}{#1}}
		\newcommand{\threecol}[1]{\multicolumn{3}{c}{#1}}
		\newcommand{\fivecol}[1]{\multicolumn{5}{c}{#1}}
		\newcommand{\ninecol}[1]{\multicolumn{9}{c}{#1}}
		\newcommand{\bsz}{Batch Size~}
		\newcommand{\certified}{{CR(\%)}}
		
		\renewcommand{\arraystretch}{1.2}
		
		\newcommand{\ccellt}[2]{\colorbox{#1}{\makebox(20,8){{#2}}}}
		\newcommand{\ccellc}[2]{\colorbox{#1}{\makebox(8,8){{#2}}}}
		\newcommand{\ccells}[2]{\colorbox{#1}{\makebox(55,8){{#2}}}}
		
		\newcommand{\temp}[1]{\textcolor{red}{#1}}
		\newcommand{\noopcite}[1]{} 
		
		\newcommand{\skiplen}{0.01\linewidth} 
		\newcommand{\rlen}{0.01\linewidth} 
		\newcolumntype{R}{>{$}r<{$}}
		\resizebox{\linewidth}{!}{
		\begin{tabular}{@{}l RRRR @{}} \toprule

				& \multicolumn{1}{c}{\text{\metric{0}}} & \multicolumn{1}{c}{\text{\metric{1}}} & \multicolumn{1}{c}{\text{\metric{2}}} &  \multicolumn{1}{c}{\text{FULL}} \\ 
				
				\midrule
				
				\multirow{1}{*}{Default} 
				& 86.9^{+4.2}_{-5.7} & 83.9^{+5.2}_{-6.9} & 82.6^{+5.7}_{-7.4} & \mathbf{68.0\pm1.7}
				\\

				\multirow{1}{*}{$\sigma$ = GELU} 
				& 82.0^{+5.3}_{-6.7} & 79.1^{+6.0}_{-7.4} & 78.4^{+6.2}_{-8.0} & 61.0\pm1.6
				\\

				\multirow{1}{*}{Pre-trained} 
				& 73.5^{+6.4}_{-7.4} & 70.0^{+7.3}_{-7.7} & 68.6^{+7.6}_{-8.3} & 49.0\pm1.4\\

				\multirow{1}{*}{Node Class.}
				& \mathbf{88.0^{+3.8}_{-5.4}} & \mathbf{85.5^{+4.6}_{-6.5}} & \mathbf{84.9^{+5.0}_{-6.6}} & 66.0\pm 1.6\\
				
				\bottomrule
		\end{tabular}}
	\end{minipage}
	\vspace{-6mm}
\end{table*}

\paragraph{Ablation studies}
We analyze the impact of our design choices and heuristics on \tool in \cref{tab:gat_additional}. Removing the node in-degree from the feature set and enumerating all possibilities during filtering still enabled exact recovery of ~30\% of graphs but caused substantial degradation. A similar, though less severe, drop occurs without the node uniqueness heuristic (\cref{sec:technical_dfs}), highlighting the key role of data-specific heuristics in strengthening the attack. We also show that \tool's stability is unaffected by architecture parameters and thresholds. As demonstrated in \cref{app:ablation}, network width and depth have minimal impact, provided the embedding dimension exceeds the number of graph nodes ($d > n$). Additionally, the $\tau$ threshold remains robust, with values between $10^{-4}$ and $10^{-2}$ yielding nearly identical filtering performance.

\paragraph{Additional experiments}
We provide additional experiments showcasing \tool's performance in different miscellaneous settings in \cref{table:extra_table}.
First, we replace the ReLU activation function in the GCN by a GELU and report that \tool achieves similar results, showing our flexibility with respect to different activations. 
Furthermore, while prior work has shown that gradient inversion becomes significantly more difficult on pre-trained models \citep{geiping2020inverting}, \tool still manages to reconstruct around 50\% of molecules exactly. 
Finally, we achieve consistently strong results in the node classification task, assuming ground-truth labels are known, which can be easily recovered using methods like \citet{zhao2020idlg}.

\section{Conclusion} 

We introduced \tool, the first gradient inversion attack for Graph Neural Networks capable of accurately recovering graphs from shared gradients. By leveraging the rank-deficiency of the GNN layers, we developed an efficient framework for extracting and filtering subgraphs of the input graph, which are iteratively combined to reconstruct the original graph. 
 
Our results showed \tool achieves an exact reconstruction rate of up to 80\% for graph classification. We introduced new metrics to evaluate partial graph reconstructions and demonstrated that \tool significantly outperforms prior work. Moreover, \tool maintains high reconstruction quality across different architectures, parameters, and settings, and can scale to much larger graphs. 

In summary, our paper is the first to demonstrate that GNN training in a federated learning setting poses data privacy risks. We believe that this is a promising initial step towards identifying these vulnerabilities and developing effective defense mechanisms.

\message{^^JLASTBODYPAGE \thepage^^J}
\clearpage
\subsubsection*{Acknowledgments}
INSAIT, Sofia University "St. Kliment Ohridski". Partially funded by the Ministry of Education and Science of Bulgaria's support for INSAIT as part of the Bulgarian National Roadmap for Research Infrastructure.

This project was supported with computational resources provided by Google Cloud Platform (GCP).

This work has been done as part of the EU grant ELSA (European Lighthouse on Secure and Safe AI,
grant agreement no. 101070617) . Views and opinions expressed are however those of the authors
only and do not necessarily reflect those of the European Union or European Commission. Neither
the European Union nor the European Commission can be held responsible for them.

The work has received funding from the Swiss State Secretariat for Education, Research and Innova-
tion (SERI).

\bibliography{references}

\begin{thebibliography}{38}
\providecommand{\natexlab}[1]{#1}
\providecommand{\url}[1]{\texttt{#1}}
\expandafter\ifx\csname urlstyle\endcsname\relax
  \providecommand{\doi}[1]{doi: #1}\else
  \providecommand{\doi}{doi: \begingroup \urlstyle{rm}\Url}\fi

\bibitem[Babai(2016)]{babai2016graph}
L{\'a}szl{\'o} Babai.
\newblock Graph isomorphism in quasipolynomial time.
\newblock In \emph{Proceedings of the forty-eighth annual ACM symposium on
  Theory of Computing}, pp.\  684--697, 2016.

\bibitem[Balunovic et~al.(2022)Balunovic, Dimitrov, Jovanovi{\'c}, and
  Vechev]{lamp}
Mislav Balunovic, Dimitar Dimitrov, Nikola Jovanovi{\'c}, and Martin Vechev.
\newblock Lamp: Extracting text from gradients with language model priors.
\newblock \emph{Advances in Neural Information Processing Systems},
  35:\penalty0 7641--7654, 2022.

\bibitem[Boenisch et~al.(2021)Boenisch, Dziedzic, Schuster, Shamsabadi,
  Shumailov, and Papernot]{cah}
Franziska Boenisch, Adam Dziedzic, Roei Schuster, Ali~Shahin Shamsabadi, Ilia
  Shumailov, and Nicolas Papernot.
\newblock When the curious abandon honesty: Federated learning is not private.
\newblock \emph{arXiv}, 2021.

\bibitem[Chu et~al.(2023)Chu, Geiping, Fowl, Goldblum, and Goldstein]{panning}
Hong-Min Chu, Jonas Geiping, Liam~H Fowl, Micah Goldblum, and Tom Goldstein.
\newblock Panning for gold in federated learning: Targeted text extraction
  under arbitrarily large-scale aggregation.
\newblock \emph{ICLR}, 2023.

\bibitem[Cui et~al.(2022)Cui, Lu, Li, and Yang]{cui2022positional}
Hejie Cui, Zijie Lu, Pan Li, and Carl Yang.
\newblock On positional and structural node features for graph neural networks
  on non-attributed graphs.
\newblock In \emph{Proceedings of the 31st ACM International Conference on
  Information \& Knowledge Management}, pp.\  3898--3902, 2022.

\bibitem[Deng et~al.(2021)Deng, Wang, Li, Shang, Liu, Rajasekaran, and
  Ding]{deng2021tag}
Jieren Deng, Yijue Wang, Ji~Li, Chao Shang, Hang Liu, Sanguthevar Rajasekaran,
  and Caiwen Ding.
\newblock Tag: Gradient attack on transformer-based language models.
\newblock \emph{arXiv preprint arXiv:2103.06819}, 2021.

\bibitem[Dimitrov et~al.(2024)Dimitrov, Baader, M{\"u}ller, and Vechev]{spear}
Dimitar~I Dimitrov, Maximilian Baader, Mark~Niklas M{\"u}ller, and Martin
  Vechev.
\newblock Spear: Exact gradient inversion of batches in federated learning.
\newblock \emph{arXiv preprint arXiv:2403.03945}, 2024.

\bibitem[Fowl et~al.(2022{\natexlab{a}})Fowl, Geiping, Reich, Wen, Czaja,
  Goldblum, and Goldstein]{decepticons}
Liam Fowl, Jonas Geiping, Steven Reich, Yuxin Wen, Wojtek Czaja, Micah
  Goldblum, and Tom Goldstein.
\newblock Decepticons: Corrupted transformers breach privacy in federated
  learning for language models.
\newblock \emph{ICLR}, 2022{\natexlab{a}}.

\bibitem[Fowl et~al.(2022{\natexlab{b}})Fowl, Geiping, Czaja, Goldblum, and
  Goldstein]{rtf}
Liam~H. Fowl, Jonas Geiping, Wojciech Czaja, Micah Goldblum, and Tom Goldstein.
\newblock Robbing the fed: Directly obtaining private data in federated
  learning with modified models.
\newblock In \emph{{ICLR}}, 2022{\natexlab{b}}.

\bibitem[Frank(2005)]{frank2005kuhn}
Andr{\'a}s Frank.
\newblock On kuhn's hungarian method—a tribute from hungary.
\newblock \emph{Naval Research Logistics (NRL)}, 52\penalty0 (1):\penalty0
  2--5, 2005.

\bibitem[Geiping et~al.(2020)Geiping, Bauermeister, Dr{\"o}ge, and
  Moeller]{geiping2020inverting}
Jonas Geiping, Hartmut Bauermeister, Hannah Dr{\"o}ge, and Michael Moeller.
\newblock Inverting gradients-how easy is it to break privacy in federated
  learning?
\newblock \emph{Advances in neural information processing systems},
  33:\penalty0 16937--16947, 2020.

\bibitem[Geng et~al.(2021)Geng, Mou, Li, Li, Beyan, Decker, and Rong]{aaai}
Jiahui Geng, Yongli Mou, Feifei Li, Qing Li, Oya Beyan, Stefan Decker, and
  Chunming Rong.
\newblock Towards general deep leakage in federated learning.
\newblock \emph{arXiv}, 2021.

\bibitem[Giles et~al.(1998)Giles, Bollacker, and Lawrence]{CiteSeer}
C.~Lee Giles, Kurt~D. Bollacker, and Steve Lawrence.
\newblock Citeseer: an automatic citation indexing system.
\newblock In \emph{Proceedings of the Third ACM Conference on Digital
  Libraries}, DL '98, pp.\  89–98, New York, NY, USA, 1998. Association for
  Computing Machinery.
\newblock ISBN 0897919653.
\newblock \doi{10.1145/276675.276685}.
\newblock URL \url{https://doi.org/10.1145/276675.276685}.

\bibitem[Hamilton et~al.(2017)Hamilton, Ying, and
  Leskovec]{hamilton2017inductive}
Will Hamilton, Zhitao Ying, and Jure Leskovec.
\newblock Inductive representation learning on large graphs.
\newblock \emph{Advances in neural information processing systems}, 30, 2017.

\bibitem[Kipf \& Welling(2016)Kipf and Welling]{kipf2016semi}
Thomas~N Kipf and Max Welling.
\newblock Semi-supervised classification with graph convolutional networks.
\newblock \emph{arXiv preprint arXiv:1609.02907}, 2016.

\bibitem[Lee et~al.(2022)Lee, Bertozzi, Kova{\v{c}}evi{\'c}, and
  Chi]{lee2022privacy}
Harlin Lee, Andrea~L Bertozzi, Jelena Kova{\v{c}}evi{\'c}, and Yuejie Chi.
\newblock Privacy-preserving federated multi-task linear regression: A one-shot
  linear mixing approach inspired by graph regularization.
\newblock In \emph{ICASSP 2022-2022 IEEE International Conference on Acoustics,
  Speech and Signal Processing (ICASSP)}, pp.\  5947--5951. IEEE, 2022.

\bibitem[Li et~al.(2022)Li, Zhang, Liu, and Liu]{ggl}
Zhuohang Li, Jiaxin Zhang, Luyang Liu, and Jian Liu.
\newblock Auditing privacy defenses in federated learning via generative
  gradient leakage.
\newblock In \emph{Proceedings of the IEEE/CVF Conference on Computer Vision
  and Pattern Recognition}, pp.\  10132--10142, 2022.

\bibitem[Lin(2004)]{rouge}
Chin-Yew Lin.
\newblock Rouge: A package for automatic evaluation of summaries.
\newblock In \emph{Text summarization branches out}, pp.\  74--81, 2004.

\bibitem[Lou et~al.(2021)Lou, Liu, Zhang, and Zheng]{lou2021stfl}
Guannan Lou, Yuze Liu, Tiehua Zhang, and Xi~Zheng.
\newblock Stfl: A temporal-spatial federated learning framework for graph
  neural networks.
\newblock \emph{arXiv preprint arXiv:2111.06750}, 2021.

\bibitem[McMahan et~al.(2017)McMahan, Moore, Ramage, Hampson, and
  y~Arcas]{mcmahan2017communication}
Brendan McMahan, Eider Moore, Daniel Ramage, Seth Hampson, and Blaise~Aguera
  y~Arcas.
\newblock Communication-efficient learning of deep networks from decentralized
  data.
\newblock In \emph{Artificial intelligence and statistics}, pp.\  1273--1282.
  PMLR, 2017.

\bibitem[Peng et~al.(2022)Peng, Wang, Dvornek, Zhu, and Li]{peng2022fedni}
Liang Peng, Nan Wang, Nicha Dvornek, Xiaofeng Zhu, and Xiaoxiao Li.
\newblock Fedni: Federated graph learning with network inpainting for
  population-based disease prediction.
\newblock \emph{IEEE Transactions on Medical Imaging}, 42\penalty0
  (7):\penalty0 2032--2043, 2022.

\bibitem[Petrov et~al.(2024)Petrov, Dimitrov, Baader, M{\"{u}}ller, and
  Vechev]{petrov2024dager}
Ivo Petrov, Dimitar~I. Dimitrov, Maximilian Baader, Mark~Niklas M{\"{u}}ller,
  and Martin~T. Vechev.
\newblock {DAGER:} exact gradient inversion for large language models.
\newblock \emph{CoRR}, abs/2405.15586, 2024.
\newblock \doi{10.48550/ARXIV.2405.15586}.
\newblock URL \url{https://doi.org/10.48550/arXiv.2405.15586}.

\bibitem[Phong et~al.(2018)Phong, Aono, Hayashi, Wang, and
  Moriai]{analyticPhong}
Le~Trieu Phong, Yoshinori Aono, Takuya Hayashi, Lihua Wang, and Shiho Moriai.
\newblock Privacy-preserving deep learning via additively homomorphic
  encryption.
\newblock \emph{{IEEE} Trans. Inf. Forensics Secur.}, \penalty0 (5), 2018.

\bibitem[Rong et~al.(2020)Rong, Bian, Xu, Xie, Wei, Huang, and
  Huang]{rong2020self}
Yu~Rong, Yatao Bian, Tingyang Xu, Weiyang Xie, Ying Wei, Wenbing Huang, and
  Junzhou Huang.
\newblock Self-supervised graph transformer on large-scale molecular data.
\newblock \emph{Advances in neural information processing systems},
  33:\penalty0 12559--12571, 2020.

\bibitem[Rossi \& Ahmed(2015)Rossi and Ahmed]{pokec}
Ryan~A. Rossi and Nesreen~K. Ahmed.
\newblock The network data repository with interactive graph analytics and
  visualization.
\newblock In \emph{AAAI}, 2015.
\newblock URL \url{https://networkrepository.com}.

\bibitem[Scarselli et~al.(2009)Scarselli, Gori, Tsoi, Hagenbuchner, and
  Monfardini]{gnn}
Franco Scarselli, Marco Gori, Ah~Chung Tsoi, Markus Hagenbuchner, and Gabriele
  Monfardini.
\newblock The graph neural network model.
\newblock \emph{IEEE Transactions on Neural Networks}, 20\penalty0
  (1):\penalty0 61--80, 2009.
\newblock \doi{10.1109/TNN.2008.2005605}.

\bibitem[Vero et~al.(2023)Vero, Balunovi{\'c}, Dimitrov, and
  Vechev]{vero2023tableak}
Mark Vero, Mislav Balunovi{\'c}, Dimitar~Iliev Dimitrov, and Martin Vechev.
\newblock Tableak: Tabular data leakage in federated learning.
\newblock In \emph{Proceedings of the 40th International Conference on Machine
  Learning}, volume 202, pp.\  35051--35083. PMLR, 2023.

\bibitem[Wen et~al.(2022)Wen, Geiping, Fowl, Goldblum, and Goldstein]{fishing}
Yuxin Wen, Jonas Geiping, Liam Fowl, Micah Goldblum, and Tom Goldstein.
\newblock Fishing for user data in large-batch federated learning via gradient
  magnification.
\newblock In \emph{{ICML}}, 2022.

\bibitem[Wu et~al.(2018)Wu, Ramsundar, Feinberg, Gomes, Geniesse, Pappu,
  Leswing, and Pande]{wu2018moleculenet}
Zhenqin Wu, Bharath Ramsundar, Evan~N Feinberg, Joseph Gomes, Caleb Geniesse,
  Aneesh~S Pappu, Karl Leswing, and Vijay Pande.
\newblock Moleculenet: a benchmark for molecular machine learning.
\newblock \emph{Chemical science}, 9\penalty0 (2):\penalty0 513--530, 2018.

\bibitem[Wu et~al.(2020)Wu, Pan, Chen, Long, Zhang, and
  Philip]{wu2020comprehensive}
Zonghan Wu, Shirui Pan, Fengwen Chen, Guodong Long, Chengqi Zhang, and S~Yu
  Philip.
\newblock A comprehensive survey on graph neural networks.
\newblock \emph{IEEE transactions on neural networks and learning systems},
  32\penalty0 (1):\penalty0 4--24, 2020.

\bibitem[Xie et~al.(2021)Xie, Ma, Xiong, and Yang]{xie2021federated}
Han Xie, Jing Ma, Li~Xiong, and Carl Yang.
\newblock Federated graph classification over non-iid graphs.
\newblock \emph{Advances in neural information processing systems},
  34:\penalty0 18839--18852, 2021.

\bibitem[Xu et~al.(2018)Xu, Hu, Leskovec, and Jegelka]{xu2018powerful}
Keyulu Xu, Weihua Hu, Jure Leskovec, and Stefanie Jegelka.
\newblock How powerful are graph neural networks?
\newblock \emph{arXiv preprint arXiv:1810.00826}, 2018.

\bibitem[Yin et~al.(2021)Yin, Mallya, Vahdat, Alvarez, Kautz, and
  Molchanov]{nvidia}
Hongxu Yin, Arun Mallya, Arash Vahdat, Jose~M. Alvarez, Jan Kautz, and Pavlo
  Molchanov.
\newblock See through gradients: Image batch recovery via gradinversion.
\newblock In \emph{{CVPR}}, 2021.

\bibitem[Zhang et~al.(2021)Zhang, Zhang, James, and Yu]{zhang2021fastgnn}
Chenhan Zhang, Shuyu Zhang, JQ~James, and Shui Yu.
\newblock Fastgnn: A topological information protected federated learning
  approach for traffic speed forecasting.
\newblock \emph{IEEE Transactions on Industrial Informatics}, 17\penalty0
  (12):\penalty0 8464--8474, 2021.

\bibitem[Zhang et~al.(2023)Zhang, Xiaoman, Sotthiwat, Xu, Liu, Zhen, and
  Liu]{zhang2023generative}
Chi Zhang, Zhang Xiaoman, Ekanut Sotthiwat, Yanyu Xu, Ping Liu, Liangli Zhen,
  and Yong Liu.
\newblock Generative gradient inversion via over-parameterized networks in
  federated learning.
\newblock In \emph{Proceedings of the IEEE/CVF International Conference on
  Computer Vision}, pp.\  5126--5135, 2023.

\bibitem[Zhao et~al.(2020)Zhao, Mopuri, and Bilen]{zhao2020idlg}
Bo~Zhao, Konda~Reddy Mopuri, and Hakan Bilen.
\newblock idlg: Improved deep leakage from gradients.
\newblock \emph{arXiv preprint arXiv:2001.02610}, 2020.

\bibitem[Zhu et~al.(2019)Zhu, Liu, and Han]{zhu2019deep}
Ligeng Zhu, Zhijian Liu, and Song Han.
\newblock Deep leakage from gradients.
\newblock \emph{Advances in neural information processing systems}, 32, 2019.

\bibitem[Zhu et~al.(2022)Zhu, Luo, and White]{zhu2022federated}
Wei Zhu, Jiebo Luo, and Andrew~D White.
\newblock Federated learning of molecular properties with graph neural networks
  in a heterogeneous setting.
\newblock \emph{Patterns}, 3\penalty0 (6), 2022.

\end{thebibliography}
\bibliographystyle{iclr2025_conference}

\message{^^JLASTREFERENCESPAGE \thepage^^J}

\ifincludeappendixx
	\clearpage
	\appendix
	
\section{Additional technical details}
\label{app:technical}
\subsection{Table of Notations}
For convenience, we add a table of notations containing brief definitons for all symbols used in our work.

\begin{table}[!hbt]
    \centering
    \caption{Table of notations used in the technical description of GRAIN.}
    \vspace{-2mm}
    \renewcommand{\arraystretch}{1.2} %
    \setlength{\tabcolsep}{4pt}      %
    \resizebox{\textwidth}{!}{
    \begin{tabular}{@{}p{0.15\textwidth}p{0.35\textwidth}@{\hskip 1cm}p{0.15\textwidth}p{0.35\textwidth}@{}}
         \toprule
         \textbf{Symbol} & \textbf{Definition} & \textbf{Symbol} & \textbf{Definition} \\
         \midrule
         $\mathcal{G} = (V, E)$ & Graph with nodes $V$ and edges $E$ & $n$ & \# of nodes in the graph \\
         
         $\boldsymbol{A^l}$ & Possibly weighted adjacency matrix at layer $l$ & $\text{dist}(v_s,v_e)$ & \# edges in shortest path connecting nodes $v_s, v_e \in V$ \\

         $\mathcal{N}^{k}_{\mathcal{G}}(v)$ & $k$-hop neighborhood in graph $\mathcal{G}$ with center node $v$ & $\mathcal{L}$ & Loss \\

         $\deg_{\mathcal{G}}(v)$ & Degree of node $v$ in graph $\mathcal{G}$ & $\deg(v)$ & Degree of node $v$ as given by its feature \\

         $\boldsymbol{X}^{i}$ & Input to the $i$th GNN layer & $\boldsymbol{X}^{i}_{v}$ & $i$-th layer input feature of node $v$ \\

         $\boldsymbol{W}^{i}$ & Weights of the $i$-th layer & $d'$ & Hidden dimension size \\

         $L$ & Number of GNN layers & $m$ & Number of features \\

         $f_i$ & Function mapping the input graph to the output of the $i$-th layer & $\tau$ & Span check distance threshold \\

         $\mathcal{F}$ & $\mathcal{F}_1 \times \dots \times \mathcal{F}_m$ - set of all possible feature combinations & $\mathcal{F}_{i}$ & Set of values for the $i$-th feature \\

         $\mathcal{T}_{l}$ & Proposal set of $l$-hop building blocks & $\mathcal{T}_{l}^{*}$ & Filtered set of $l$-hop building blocks \\

         $\mathcal{T}_{B}^{*}$ & Final set of filtered building blocks & $\sigma$ & Activation function \\

         $\Delta_{\mathcal{G}}$ & Distance between the gradients of $\mathcal{G}$ and observed gradients & $d_{best}$ & Gradient distance of the best reconstructed graph \\

         GSM-N($\mathcal{G}, \hat{\mathcal{G}}$) & Similarity between N-hop neighborhoods of $\mathcal{G}$ and $\hat{\mathcal{G}}$ & $\mathcal{G}_{best}$ & The best reconstructed graph. \\

         \bottomrule
    \end{tabular}}
    \vspace{-5mm}
    \label{tab:table_of_notations}
\end{table}

\subsection{Deferred proofs}
\label{app:proofs}
\subsubsection{Span Check Proof}
Here we show the proof of \cref{thm:spancheck_ours}, which we restate here for convenience:
\spcheck*
\begin{proof}

For notational clarity, we omit the layer index $l$ in our proof. We separate the proof in 3 steps:
\begin{itemize}
	\item Step 1: $\mX_i^T \in \ColSpan(\grad{W})$ if and only if there is a vector $\alpha_i$, such that $ \grad{\mY}\alpha_i = e_i$, where $e_i$ is the $i$-th standard basis vector.
	\item Step 2: There is a vector $\alpha_i$, such that $\grad{\mY}\alpha_i = e_i$ if and only if $\NullSpace(\hat{\grad{\mY_i}})\not\subseteq\NullSpace(\grad{\mY_i})$.
	\item Step 3: $\NullSpace(\hat{\grad{\mY_i}})\not\subseteq\NullSpace(\grad{\mY_i})$ is equivalent to $\grad{\mY_i}\notin\RowSpan(\hat{\grad{\mY_i}})$.
\end{itemize}

\paragraph{Step 1 ($ \mX_i^T \in \ColSpan(\grad{\mW}) \iff \exists \alpha_i. \grad{\mY}\alpha_i = e_i$):}
\begin{itemize}
	\item[($\Rightarrow$)] First, we begin by expressing $\grad{\mW}$ using the following common result from \citet{petrov2024dager}:
	
	$$\grad{\mW} = \mX^T\grad{\mY},$$
	implying that $\mX_i^T \in \ColSpan(\mX^T\grad{\mY}) = \RowSpan(\grad{\mY}^T\mX)$. This can be rewritten as $\exists \alpha_i. \alpha_i^T \grad{\mY}^T\mX = \mX_i^T$. Assuming $\mX \in \mathbb{R}^{n\times d}$ is full-rank, then there exists a right-inverse $\mX^{-R}$, as $\rank(\mX) = n < d$. 
	$$\alpha_i^T \grad{\mY}^T\mX\mX^{-R} = \mX_i^T\mX^{-R} \Rightarrow \alpha_i^T \grad{\mY}^T = e_i^T \iff \grad{\mY}\alpha_i = e_i$$

	It is notable that $\mX$ not being full-rank still allows for all nodes with feature vectors in $\mX$ to pass the span check, however it is possible that some hallucinated inputs might also pass the check.

	\item[($\Leftarrow$)]
		\begin{align*}
			\grad{\mY}\alpha_i &= e_i \iff \\
			\alpha_i^T \grad{\mY}^T &= e_i^T \iff \\
			\alpha_i^T \grad{\mY}^T X &= X_i^T \\
		\end{align*}
		Thus $\mX_i^T \in \RowSpan(\grad{\mY}^T\mX) = \ColSpan(\mX^T\grad{\mY}) = \ColSpan(\grad{\mW})$. Therefore $\mX_i^T \in \ColSpan(\grad{\mW})$.

\end{itemize}

\paragraph{Step 2 ($\exists\alpha_i. \grad{\mY}\alpha_i= e_i\iff\NullSpace(\hat{\grad{\mY_i}})\not\subseteq\NullSpace(\grad{\mY_i})$):}

First, for both directions of the proof, we can separate $\grad{\mY}\alpha_i = e_i$ into 2 different requirements:

\begin{equation}
	\label{eq:proof5_1_eq1}
	\hat{\grad{\mY_i}}\alpha_i = \bm{0}
\end{equation}

\begin{equation}
	\label{eq:proof5_1_eq2}
	\grad{\mY_i}\alpha_i = 1
\end{equation}

\begin{itemize}
	\item[($\Rightarrow$)] Assuming the existence of $\alpha_i$ with $ \grad{\mY_i}\alpha_i = e_i$, we know that \cref{eq:proof5_1_eq1} and \cref{eq:proof5_1_eq2} hold. 
	It is evident with $\alpha_i \in \NullSpace(\hat{\grad{\mY_i}})$ but $\alpha_i \notin \NullSpace({\grad{\mY_i}})$ that $\NullSpace(\hat{\grad{\mY_i}})\not\subseteq\NullSpace(\grad{\mY_i})$.%
	
	\item[($\Leftarrow$)] First of all, we note that $\grad{\mY_i} \in \mathbb{R}^{1\times d}$ has rank 1, as $\grad{\mY_i}$ contains a non-zero entry under normal training conditions. Therefore, $\grad{\mY_i}$ has $\nullity(\grad{\mY_i}) = n - 1$ due to the rank-nullity theorem. $\NullSpace(\hat{\grad{\mY_i}})\not\subseteq\NullSpace(\grad{\mY_i})$ implies that there exists an $\alpha_i' \in\NullSpace(\hat{\grad{\mY_i}})$, such that $\alpha_i'\notin\NullSpace(\grad{\mY_i})$ (since $\nullity(\grad{\mY_i}) = n - 1<n$ this set is non-empty). For that $\alpha_i$, the following hold:
$$\hat{\grad{\mY_i}}{\alpha_i'} = \bm{0}$$
$$\grad{\mY_i}{\alpha_i'} = c$$
Therefore, if we take $\alpha_i = \frac{1}{c}\alpha_i'$, $\alpha_i$ would satisfy both (2) and (3), giving us a valid solution.
\end{itemize}

\paragraph{Step 3: ($\NullSpace(\hat{\grad{\mY_i}})\not\subseteq\NullSpace(\grad{\mY_i})\iff\grad{\mY_i}\notin\RowSpan(\hat{\grad{\mY_i}})$)}

First of all, the statement is equivalent to negating both sides, or $\NullSpace(\hat{\grad{\mY_i}})\subseteq\NullSpace(\grad{\mY_i})\iff\grad{\mY_i}\in\RowSpan(\hat{\grad{\mY_i}}) $, which can be shown by the following steps:

\begin{align*}
	\NullSpace(\hat{\grad{\mY_i}})\subseteq\NullSpace(\grad{\mY_i}) 
	&\iff 
		\NullSpace(\grad{\mY_i})^C\subseteq\NullSpace(\hat{\grad{\mY_i}})^C \\
	&\iff
		\RowSpan(\grad{\mY_i})\subseteq\RowSpan(\hat{\grad{\mY_i}}) \\
	&\iff 
		\grad{\mY_i}\in\RowSpan(\hat{\grad{\mY_i}})
\end{align*}

Here we used that the complenetary subspace of the null space of matrix is the rowspan of the matrix $\NullSpace(M)^C = \RowSpan(M)$. The last step follows from that the fact that $\grad{\mY_i}$ is a single common vector, and therefore all vectors in $\ColSpan(\grad{\mY_i})$ are of the form $\lambda \grad{\mY_i}$. This concludes our proof.

\end{proof}

We extend this by presenting the proof for \cref{clr:spancheck_full_rank}:
\fullrank*
\begin{proof}
	For notational clarity, we omit the layer index $l$ in our proof.

	From $\mZ = \mA\mY$, by applying the following common result from \citet{petrov2024dager} (replacing $\mW$ with $\mY$, and $\mY$ with $\mZ$), we obtain:
	$$\grad{\mY} = \mA^T\grad{\mZ}$$
	Therefore, $\grad{\mY_i} = \mA^T_i\grad{\mZ}$. We will now prove that if $\grad{\mZ}$ is full-rank, then $\mA^T_i\notin\ColSpan(\hat{\mA_i})\iff\grad{\mY_i}\notin\RowSpan(\hat{\grad{\mY_i}})$. This is equivalent to proving the converse, or $\mA^T_i\in\ColSpan(\hat{\mA_i})\iff\grad{\mY_i}\in\RowSpan(\hat{\grad{\mY_i}})$
	
	($\Rightarrow$) $\mA^T_i\in\ColSpan(\hat{\mA_i})$ implies that there exist coefficients $\alpha_1, \alpha_2, \cdots, \alpha_{n}$, such that:
	$$\mA^T_i = \sum_{j=\{1,2,\dots,N\}\setminus\{i\}}\alpha_j\mA_j^T,$$
	
	Multiplying both sides by $\grad{\mZ}$ gives:
	$$\mA^T_i\grad{\mZ} = \sum_{j=\{1,2,\dots,N\}\setminus\{i\}}\alpha_j\mA_j^T\grad{\mZ}\iff\grad{\mY_i} = \sum_{j=\{1,2,\dots,N\}\setminus\{i\}}\alpha_j\grad{\mY_j}\iff\grad{\mY_i}\in\RowSpan(\hat{\grad{\mY_i}})$$
	
	($\Leftarrow$) We can similarly rewrite $\grad{\mY_i}\in\RowSpan(\hat{\grad{\mY_i}})$ as:
	$$\grad{\mY_i} = \sum_{j=\{1,2,\dots,N\}\setminus\{i\}}\alpha_j\grad{\mY_j}\iff\mA^T_i\grad{\mZ} = \sum_{j=\{1,2,\dots,N\}\setminus\{i\}}\alpha_j\mA_j^T\grad{\mZ}$$
	
	As $\grad{\mZ}$ is full-rank, then there exists a right-inverse $\grad{\mZ}^{-R}$. Multiplying on both sides gives:
	$$\mA^T_i\grad{\mZ}\grad{\mZ}^{-R} = \sum_{j=\{1,2,\dots,N\}\setminus\{i\}}\alpha_j\mA_j^T\grad{\mZ}\grad{\mZ}^{-R}\iff$$
	$$\mA^T_i = \sum_{j=\{1,2,\dots,N\}\setminus\{i\}}\alpha_j\mA_j^T\iff \mA^T_i\in\ColSpan(\hat{\mA_i})$$
	
	From here, we can now apply \cref{thm:spancheck_ours}, which gives us that $\mX_i \in \ColSpan(\grad{\mW})$ if and only if $\mA^T_i\notin\ColSpan(\hat{\mA_i})$. If $\mA$ is full-rank, then $\mA^T_i\notin\ColSpan(\hat{\mA_i})$ for all $i = 1, 2, \cdots, n$, implying that $\mX_i \in \ColSpan(\grad{\mW})$ for all $i$. This concludes the proof.
\end{proof}

\subsubsection{Embedding Recovery Proof}
In this section, first we formally state \cref{assmp:dependence} on which \cref{thm:sbgrprop} is based. Next, using it we provide a proof for \cref{thm:sbgrprop}. Finally, we prove that \cref{assmp:dependence} is satisfied for common GNN architectures, such as GCNs and GATs. 

For simplicity, in the rest of the section we will denote the node of $\gG$ corresponding to the the $i^{\text{th}}$ row of the embedding matrices $\mX^l_i$ with $v_i$. 
\begin{assumption}{}
	\label{assmp:dependence}
	The output corresponding to the node $v_i$ of the $l$-th layer of the GNN, $X^{l+1}_i$, is independent on $X^{l}_j$ for $\forall v_j \notin \gN_{\gG}^1(v_i)$.
\end{assumption}
The assumption intuitively states that for most popular GNNs the embeddings of a node only depends on the embeddings of its neighbourhood on the previous layer. Using this, we now present the proof of \cref{thm:sbgrprop}, restating it for convinience:
\sbgrprop*
\begin{proof}

	First we will show that the embedding $\mX^l_i$ is independent of the embeddings of any nodes that do not belong in the $l$-hop neighborhood of $v$. For $l=1$ this follows immediately from \cref{assmp:dependence}. For $l=2$, we know that $\mX^2_i$ is only dependent on the embeddings $\mX^1_k$ of the neighboring nodes $v_k\in\gN^1_\gG(v)$ by applying \cref{assmp:dependence}. However, the embeddings $\mX^1_k$ themselves only depends on the embeddings of the nodes in the $1$-hop neighborhood of $v_k$. Therefore, the only nodes that can influence $\mX^2_i$ are $v$, its neighbors, and the neighbors of its neighbors, which exactly comprise the $2$-hop neighborhood of $v$. Similarly, we can show that $\mX^l_i$ is independent of the embeddings of any nodes that are not part of the $l$-hop neighborhood of $v$ by induction.

	We now take the full graph $\gG$ and remove all edges starting or ending at nodes not in $\gG^l_v$, and set their feature values to $0$, obtaining $\hat{\gG}$. As shown above, this implies the output of $f_{l-1}$ for $\gG$ at $v$ is the same as the one for $\hat{\gG}$. Note that applying the network on $\hat{\gG}$ is the same as applying it to $\gG^l_v$ except for few the additional embeddings produced for the disconnected nodes. By reindexing we arrive at our formula.

	\end{proof}

	Finally, we prove our assumption \cref{assmp:dependence} holds for all GNN architectures considered in this paper:
	
	\begin{theorem} \cref{assmp:dependence} holds for GCNs \end{theorem}
	\begin{proof}
		For GCN models, the formula for computing the $i^\text{th}$ embedding at layer $l+1$, $\mX^{l+1}_i$, can be expressed as:
		\begin{equation*}
			\mX^{l+1}_i = \sigma\left(\sum_{v_j \in \mathcal{N}^1_{\gG}(v_i)} \mA^{l}	_{i,j} \mathbf{W}^{l} \mX^{l}_j\right)
		\end{equation*}
		where $\mA^{l}_{i,j}$ is the normalized strength of the connection between $v_i$ and $v_j$ at layer $l$ given by:
		\begin{equation*}
			\mA^{l}_{i,j} = \frac{1}{\sqrt{deg(v_i)deg(v_j)}}.
		\end{equation*}
		Assuming the degrees of the nodes $v_j\in\mathcal{N}^1_{\gG}(v_i)$ are known, $\mX^{l+1}_i$ is independent on $\mX^{l}_j$ for $v_j\notin\mathcal{N}^1_{\gG}(v_i)$.
		We note that for a correctly recovered $l$-hop neighbourhood around $v$, the degrees of all nodes in the $(l-1)$-hop neighbourhood around $v$ can be recovered exactly, while for the remaining nodes one can provide a sound lower bound for their degree. Assuming a sound guess on the upper bound on the largest degree in the graph, this allows to recover the degrees of the remaining nodes via enumeration starting from these computed lower bounds. Further, if the input features $\mX^0_i$ of $v_i$ contain information about the node degree of  $v_i$, this information can be easily incorporated to reduce the computational overhead of the enumeration.
	\end{proof}
	\begin{theorem} \cref{assmp:dependence} holds for GATs \end{theorem}
	\begin{proof}
		For GAT models, the formula for computing the $i^\text{th}$ embedding at layer $l+1$, $\mX^{l+1}_i$, can similarly be expressed as:
		\begin{equation*}
			\mX^{l+1}_i = \sigma\left(\sum_{v_j \in \mathcal{N}^1_{\gG}(v_i)} \mA^{l}	_{i,j} \mathbf{W}^{l} \mX^{l}_j\right)
		\end{equation*}
		where $\mA^{l}_{i,j}$ is the attention coefficient between $v_i$ and $v_j$ at layer $l$ given by:
		\begin{equation*}
		\mA^{l}_{i,j} = \frac{\exp\left(\text{LeakyReLU}\left((\mathbf{a}^	{l})^T [\mathbf{W}^l \mX^{l}_i \, \| \, \mathbf{W}^{l} \mX^{l}_j]	\right)\right)}{\sum_{v_k \in \mathcal{N}^1_{\gG}(v_i)} \exp\left(\text	{LeakyReLU}\left((\mathbf{a}^{l})^T [\mathbf{W}^{l} \mX^{l}_i \, \| 	\, \mathbf{W}^{l} \mX^{l}_k]\right)\right)}
		\end{equation*}
		and $a^l$ is a vector of attention parameters.  
		
		As $\mX^{l+1}_i$ only depends on $\mW^l$,$\mA^{l}_{i,j}$ and $\mX^{l}_j$ to be calculated for $v_j\in\mathcal{N}^1_{\gG}(v_i)$, and $\mA^{l}_{i,j}$ itself also only depends on $\mX^{l}_j$ and $\mW^l$, we show that $\mX^{l+1}_i$ is independent on $\mX^{l}_j$ for $v_j\notin\mathcal{N}^1_{\gG}(v_i)$.
	\end{proof}
\subsection{Gluing algorithm}
\label{app:gluing}
In order to establish the gluing operation, we need to first define how we determine which nodes in $\gG$ and $\gG^B$ can potentially match, so that the building block is glued correctly. This can be achieved recursively, by ensuring that:
\begin{itemize}
	\item For every pair of matched nodes, their features must be exactly equal.
	\item The center of $\gG^B$ - $c^B$ matches the center of attaching $c$.
	\item For a node $v\in V$ of $\gG$ and its match $v_B\in V^B$, every one of its neighbors $v'$ ($(v,v')\in E$) has to match a neighbor $v'_B$ of $v_B$ ($(v_B,v'_B)\in E^B$).
\end{itemize}

This recursive definition can be satisfied by exploring the possible matchings from the center $c$ outwards. In particular, we explore the $i$-hop neighborhoods one by one for $i=1,2,\cdots,l$. We try all possible matchings $v_B$ for the given node $v$, such that their features match, and the matches of its neighbors (in particular those which have already been traversed) are also connected to $v_B$ (Line \ref{alg:extend_match}).

After all possible matchings are generated, it is easy to substitute any nodes in the original graph with those in the building block, in order to obtain the set of all possible gluings (Lines \ref{alg:new_edges_start}--\ref{alg:new_edges_end})
\begin{algorithm}[H]
	\caption{Gluing a graph with a building block}
	\label{alg:gluing}
	\begin{algorithmic}[1]
		\Function{Glue}{$\gG$, $\gG^B$, $c$, $l$ } 
		\State $\hat{M}\gets\{\{(c, c^B)\}\}$

		\For{$i=1,2,\cdots,l$}		\Comment{Generate all correct matchings}
			\For{$v\in\{v|\dist{c,v}=i\}$}
				\State $\hat{M}_{\text{new}}\gets\{\}$
				\For{$\mathcal{M}\in\hat{M}$}
				\State $\hat{M}_{\text{new}}\gets\hat{M}_{\text{new}}\times\{(v, v_B)|\mX_{v} = \mX_{v_B} \land \cap_{(v, v') \in E,(v', v'_B)\in \mathcal{M}} (v_B, v'_B) \in E^B\}$\label{alg:extend_match}
				\EndFor
				\State $\hat{M}\gets\hat{M}_{\text{new}}$
			\EndFor
		\EndFor

		\State $\mathbb{G}\gets\{\}$
		\For{$\mathcal{M}\in\hat{M}$} \Comment{Transform matchings into valid graphs}
			\State $\hat{V}\gets V\setminus\mathcal{N}^l_\gG(c) \cup V_B$
			\State $\hat{E}\gets E^B$
			\For{$(v, v')\in E$}\label{alg:new_edges_start}
				\State $v = v_B \text{ if } \exists v_B. (v, v_B)\in \mathcal{M} \text{ else } v$
				\State $v' = v'_B \text{ if } \exists v'_B. (v', v'_B)\in \mathcal{M} \text{ else } v'$
				\State $\hat{E} \gets \hat{E} \cup \{(v, v')\}$
			\EndFor\label{alg:new_edges_end}
			
			\State $\mathbb{G} \gets \mathbb{G} \cup \{(\hat{V}, \hat{E})\}$
		\EndFor
		\State \Return $\mathbb{G}$
		\EndFunction
	\end{algorithmic}
\end{algorithm}

\subsection{Feature-by-feature filtering algorithm}
\label{app:fbyf}

	\begin{algorithm}[H]
		\caption{Filtering nodes feature-by-feature}
		\label{alg:fbyf}
		\begin{algorithmic}[1]
			\Function{FilterNodes}{$\{\mathcal{F}_i\}_{i\in\{1,\cdots,m\}}$ ,$\grad{\mW_0}$, $\tau$} 
			\State \{$\mathcal{T}^{*}_{0,i}\}_{i\in\{1,\cdots,m\}}, d_{sum}\leftarrow\{\emptyset\}_{i\in\{1,\cdots,m\}}$, $|\mathcal{F}_1|$
			\For{$k\in \{1,\cdots,m\}$}
				\State $\mathcal{T}_{0,k}^*\leftarrow\mathcal{T}_{0,k-1}^* \times \mathcal{F}_{k}$
				\State $d_{sum} \leftarrow d_{sum}+\lvert\mathcal{F}_{k}\rvert$
				\If{$d_{sum} > \rank(\grad{\mW_0})$} \Comment{This is a requirement for the filtering to work}
					\State $\mathcal{T}_{0,k}^*\leftarrow$ \Call{Filter}{$\mathcal{T}_{0,k}^*, \grad{\mW_0}[:d_{sum}],\tau,\lambda v. \mX_v^0$}
				\EndIf
			\EndFor
			\State \Return $\mathcal{T}^{*}_{0,m}$ 
			\EndFunction
		\end{algorithmic}
	\end{algorithm}

Here we describe how we build and filter 0-hop neighborhoods, i.e. the single node client feature vectors. While DAGER directly enumerates and filters all possible text input feature vectors, for many graph data applications this procedure is impractical. To this end, we instead recover the node features one-by-one, as described next.

For the first feature, we generate all partial input vectors associated with the values in $\mathcal{F}_1$ and then filter them to create the set of consistent partial input vectors $\mathcal{T}_{0,1}^*$. To do so, we apply our span-check on the truncated input layer gradients $\frac{\partial \mathcal{L}}{\partial W^0}\left[:|\mathcal{F}_1|\right]$ corresponding to the input entries of the first feature. This is possible since, by \cref{thm:spancheck_ours}, the truncated $i$-th row of $\mX$, $\mX\left[i, :|\mathcal{F}_1|\right]$, is in $\ColSpan(\frac{\partial \mathcal{L}}{\partial W^0}\left[:|\mathcal{F}_1|\right])$. 
For each subsequent feature, we similarly filter the set of vectors corresponding to combinations between a partially recovered vector in $\mathcal{T}_{0,i-1}^*$ and a value in $\mathcal{F}_i$ by applying \cref{thm:spancheck_ours} to obtain the set of consistent partial input vectors up to feature $i$ $\mathcal{T}_{0,i}^*$. Finally, we obtain $\mathcal{T}_0^* = \mathcal{T}_{0,m}^*$.
A pseudocode for this approach is provided in \cref{alg:fbyf}.

\subsection{Structure-based filtering algorithm}
\label{app:str_filter}
\begin{algorithm}[H]
    \caption{Structure filtering algorithm}
    \label{alg:structure_filter}
    \begin{algorithmic}[1]
        \Function{StructureFilter}{$\mathcal{T}_{L}^*$, $\grad{\mW}$} 
        \For{$\gG\in\mathcal{T}^{*}_{L}$}
            \If{$\Delta_\gG == 0$}
                \State \Return $\gG$
            \EndIf
            \State $CanGlue\gets True$
            \For{$v\in V$}    
                \If{$!\exists\gG^B\in\mathcal{T}^{*}_{L}.\glue{\gG, \gG^B, v}$}
                    \State $CanGlue \gets False$
                    \State \textbf{break}
                \EndIf
            \EndFor
            \If{$CanGlue$}
                \State $\mathcal{T}^*_B\gets\mathcal{T}^*_B\cup\{\gG\}$
            \EndIf
        \EndFor
        \State \Return $\mathcal{T}_{B}^*$
        \EndFunction
    \end{algorithmic}
\end{algorithm}

\subsection{Depth-First Search implementation}
We present the pseudocode for branching the DFS tree:
\label{app:branching}
\paragraph{Cycles exploration}
After every step of gluing possible building blocks from $\mathcal{T}_{B}^{*}$ to $\gG_{\text{curr}}$ at node $v$ (line \ref{glue_bb} in \cref{alg:dfs_br}), we enumerate all sets $S$ of pairs $(v_1 \in \gG_{\text{curr}}, v_2 \not\in \gG_{\text{curr}})$ of vertices of any graph $\gG' \in \mathbb{G}$ such that the features of $v_1$ and $v_2$ match (line \ref{enum_pairs} in \cref{alg:dfs_br}). For every such set $S$, we additionally consider for exploration the graph $\hat{\gG'} = \overlap{\gG', S}$ created by overlapping each pair of vertices in $S$ (line \ref{add_coincided} in \cref{alg:dfs_br}). Intuitively, this allows us to reconstruct any cycle of the input graph. All valid overlaps, including those with $S=\emptyset$, where the graph remains unchanged, are then added as branches of the search space (line \ref{branch} in \cref{alg:dfs}).
	\begin{algorithm}[H]
		\caption{Generate graphs for branch}
		\label{alg:dfs_br}
		\begin{algorithmic}[1]
			\Function{Branch}{$\mathcal{T}^{*}_{B}$, $\gG_{\text{curr}}$, $v$}
    		\State  $\mathbb{G}_{\text{new}} \leftarrow \{\}$
    		\State 
    		\For{$\gB \in \mathcal{T}^{*}_{B}$}
    			\State $\mathbb{G} \leftarrow \glue{\gG_{\text{curr}}, \gB, v}$ \label{glue_bb}
    			\If{$\mathbb{G} \neq \emptyset$}
    				\State $\mathbb{G}_{\text{new}} \leftarrow \mathbb{G}_{\text{new}} \cup \mathbb{G}$
    				\For{$\gG' \in \mathbb{G}$}
    					\For{$S\subseteq \{V(\gG')\setminus V(\gG_{\text{curr}})\}\times V(\gG_{\text{curr}})$}{\label{enum_pairs}}
    						\State$\mathbb{G}_{\text{new}}\leftarrow \mathbb{G}_{\text{new}} \cup \{\overlap{\gG', S}\}$ \label{add_coincided}
    					\EndFor
    				\EndFor
    			\EndIf
    		\EndFor
    		\State
    		\State\Return $\mathbb{G}_{\text{new}}$
    		\EndFunction
		\end{algorithmic}
	\end{algorithm}

\paragraph{Node overlaping}
In this paragraph, we describe how the overlaping function (used in line \ref{add_coincided} in \cref{alg:dfs_br}) works. Given a graph $\gG'$ and a set of pairs of nodes $S$, if some pair has two nodes with different node feature vectors, the function returns an empty graph. Otherwise, for each pair $(v_1,v_2)\in S$ we modify the current graph by deleting $v_2$ and connect add an edge between $v_1$ and any node that was previously a neighbor of $v_2$ but not $v_1$. After doing these overlaps for every pair of nodes in $S$, the new graph is returned.
\section{The GSM metric}
\label{app:gsm_disc}
We designed the GSM set of metrics so that they satisfy the following three qualities:
\begin{itemize}
	\item The metric should be efficiently computable in polynomial time
	\item It should capture both structural and feature-wise information
	\item Isomorphic graphs should be guaranteed to achieve a 100$\%$ score
\end{itemize}

To this end, we define $\text{\metric{N}}_F(\gG, \hat{\gG})$ under a set of functions $F = \{F_k\}_{k=1}^{N}$, where for all $k$   $F_{k}:\gG\rightarrow\mathbb{R}^{|\mathcal{V}|\times d}$ is a function that aggregates the feature vectors for each $k$-hop neighborhood. This allows us to measure the similarites in features across increasingly larger subgraphs, which capture the structure around each node. In our case we utilise a randomly initialised $\geq k$-layer GCN to achieve such a mapping.

We note that a precise evaluation of the metric requires for us to match the 2 graphs as accurately as possible. Since exact matching of graphs has no known polynomial-time algorithm \citep{babai2016graph}, we match the graph nodes by applying the Hungarian matching algorithm \citep{frank2005kuhn} for minimizing a cost function $\mC$ that captures the feature difference across (0-5)-hop neighborhoods:
$$\mC_{ij} = \sum^2_{k=0}\sum^d_{m=1} (F_k(\gG) - F_k(\hat{\gG}))_m^2$$

We can hence define:
$$\text{\metric{N}}_F(\gG, \hat{\gG}) = \begin{cases}
	\text{F1-Score}(F_N(\gG), F_N(\hat{\gG}))& \text{if }F_N(\gG)\text{ - discrete}\\
	R^2(F_N(\gG), F_N(\hat{\gG}))& \text{if }F_N(\gG)\text{ - continuous}\\
\end{cases}$$

First of all, the NP-complete nature of the subgraph isomorphism problem makes it difficult to do any subgraph or full-graph matching, which we tackle by utilising the hidden states of a GCN to create an approximate matching between nodes. The method described above ensures that for isomorphic graphs, the correct k-hop neighborhoods will be matched and the \metric{k} metric will achieve a perfect score.

The second requirement is rarely satisfied by metrics defined in the literature (i.e. the edit distance), as comparison studies on coloured graphs are limited.
Our solution to these problems was inspired by the ROUGE set of metrics~\citep{rouge}, used for evaluation of textual similarity. Instead of comparing sequences such as unigrams or bigrams like ROUGE, we instead compute continuous properties of graphs on the scale of different k-hop neighbourhoods. This rationale allows us to compare both node-based and structural properties using a simple methodology.
\section{Additional Experiments}

\label{app:experiments}
Here we present additional experiments that are not part of the main text.

\subsection{Human evaluation for the GSM set of metrics}
We performed a human evaluation, where 3 experts in Graph Theory and Chemistry were shown 120 sample reconstructions of molecules, as given by DLG and \tool. The samples were shuffled, and the participants were tasked to assign a score from 0 to 10, with the following instructions:

"Thank you for agreeing to participate in this study on the quality of graph reconstructions! We have gathered a set of graphs, coupled with the best-effort reconstruction. Please give each pair a score of 0-10, where 0 is a complete lack of similarity, and 10 is a perfect match. When assigning a score, take into account the \emph{structure} of the two graphs, as well as the \emph{atom type} for matching atoms, and also be wary that 2 graphs might be \emph{isomorphic}, but have different pictures. Please disregard the connections between atoms, as the methods we used do not recover any edge properties. Give your, as best as possible, score on how similar the graphs are with respect to these properties."

Reconstructed samples from both \tool and DLG were shuffled and anonymized before being presented to the participants. We report the average scores for each algorithm, multiplied by a factor of 10 to match the order of magnitude of the GRAPH metrics, and present the results in \cref{tab:study}. We observe very good correlation between our metrics and the reported human scores, even though our metrics are slightly more lenient to completely wrong reconstructions, compared to the evaluators. This leniency provides a slight advantage to the baseline attacks when measured using our metrics, as the baselines fail catastrophically more often.

Based on these studies, we also show in \cref{tab:study_comparison} that our partial reconstructions are deemed more significant than what the metric suggests, likely meaning that there are examples which present significant information leakage. In contrast, high-scoring examples from the DLG attacks have been rated as essentially uninformative.

\begin{table}[!hbt]
    \centering
    \caption{Score discrepancy examples between human evaluators and the GRAPH set of metrics. G-1 stands for the GRAPH-1 metric.}
    \label{tab:study_comparison}
    \resizebox{\linewidth}{!}{%
    \begin{tabular}{>{\centering\arraybackslash}m{0.12\textwidth} >{\centering\arraybackslash}m{0.12\textwidth} c c >{\centering\arraybackslash}m{0.12\textwidth} >{\centering\arraybackslash}m{0.12\textwidth} c c}
        \textbf{GT} & \textbf{GRAIN} & \textbf{G-1} & \textbf{Study} & \textbf{GT} & \textbf{DLG} & \textbf{G-1} & \textbf{Study} \\
        \hline
        \includegraphics[width=0.12\textwidth]{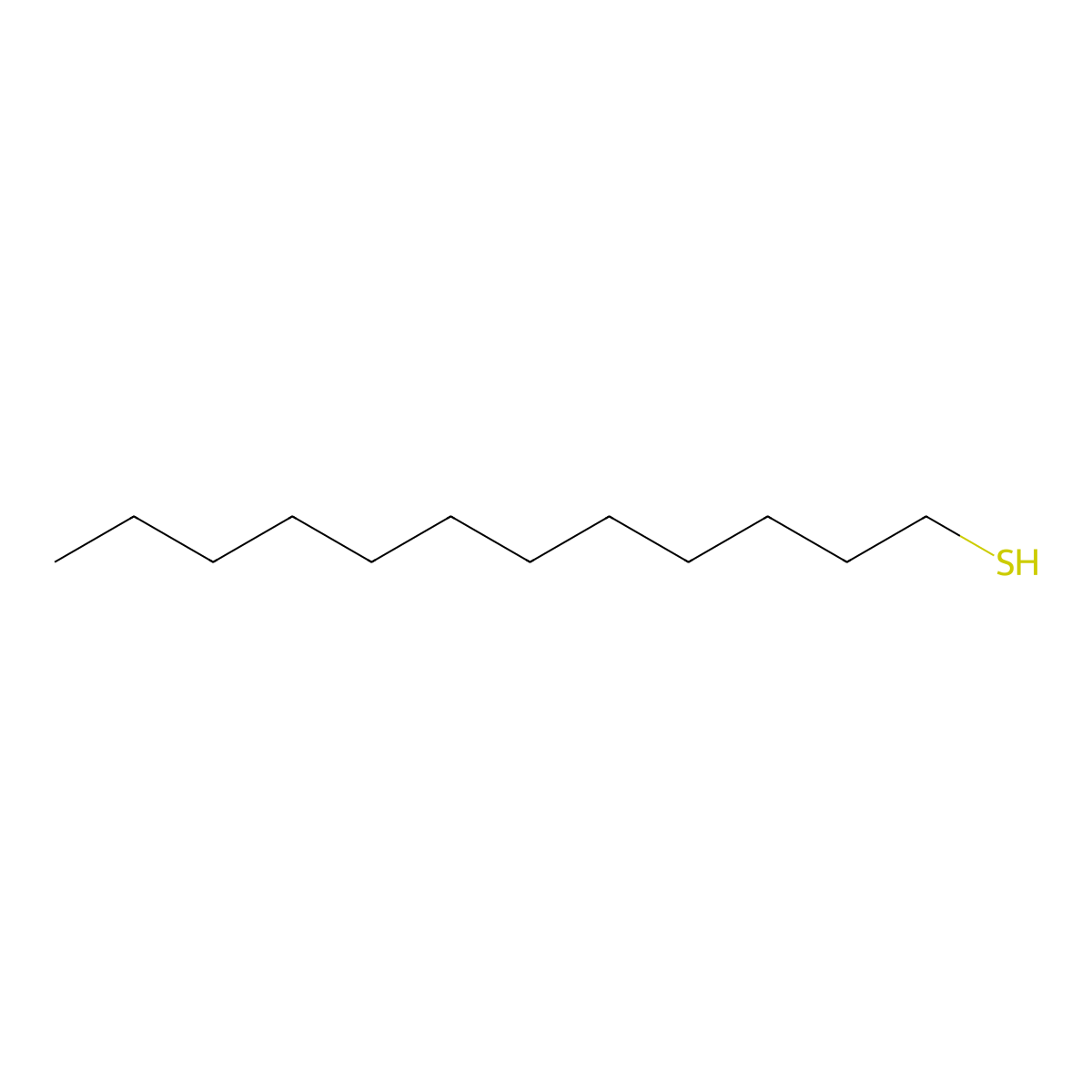} & 
        \includegraphics[width=0.12\textwidth]{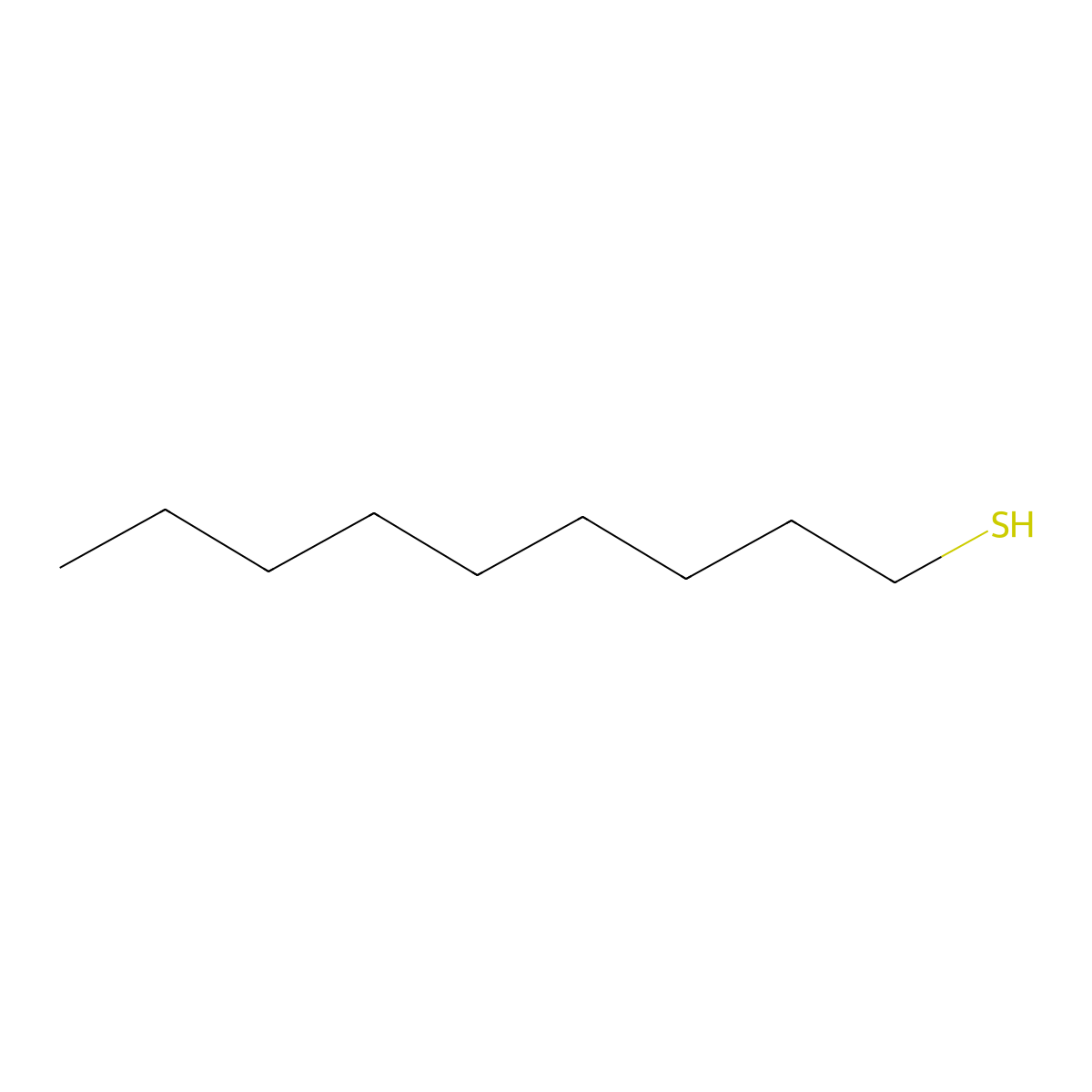} & 
        62.0 & 
        93.3 & 
        \includegraphics[width=0.12\textwidth]{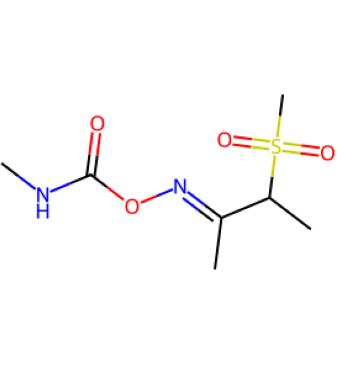} & 
        \includegraphics[width=0.12\textwidth]{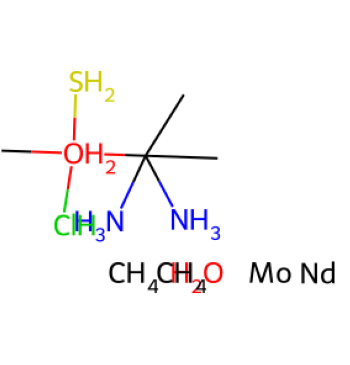} & 
        52.7 & 
        10.0 \\
        \includegraphics[width=0.12\textwidth]{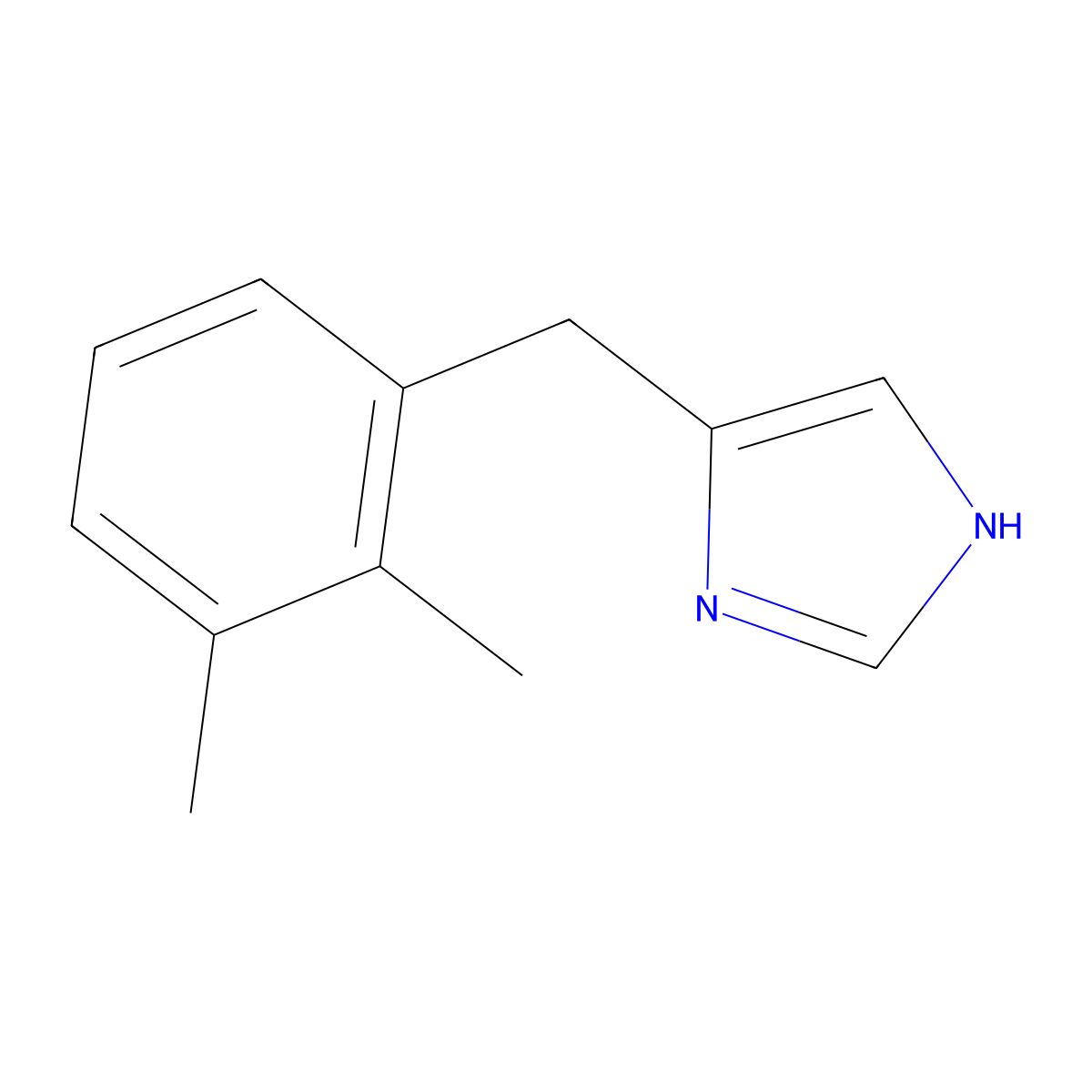} & 
        \includegraphics[width=0.12\textwidth]{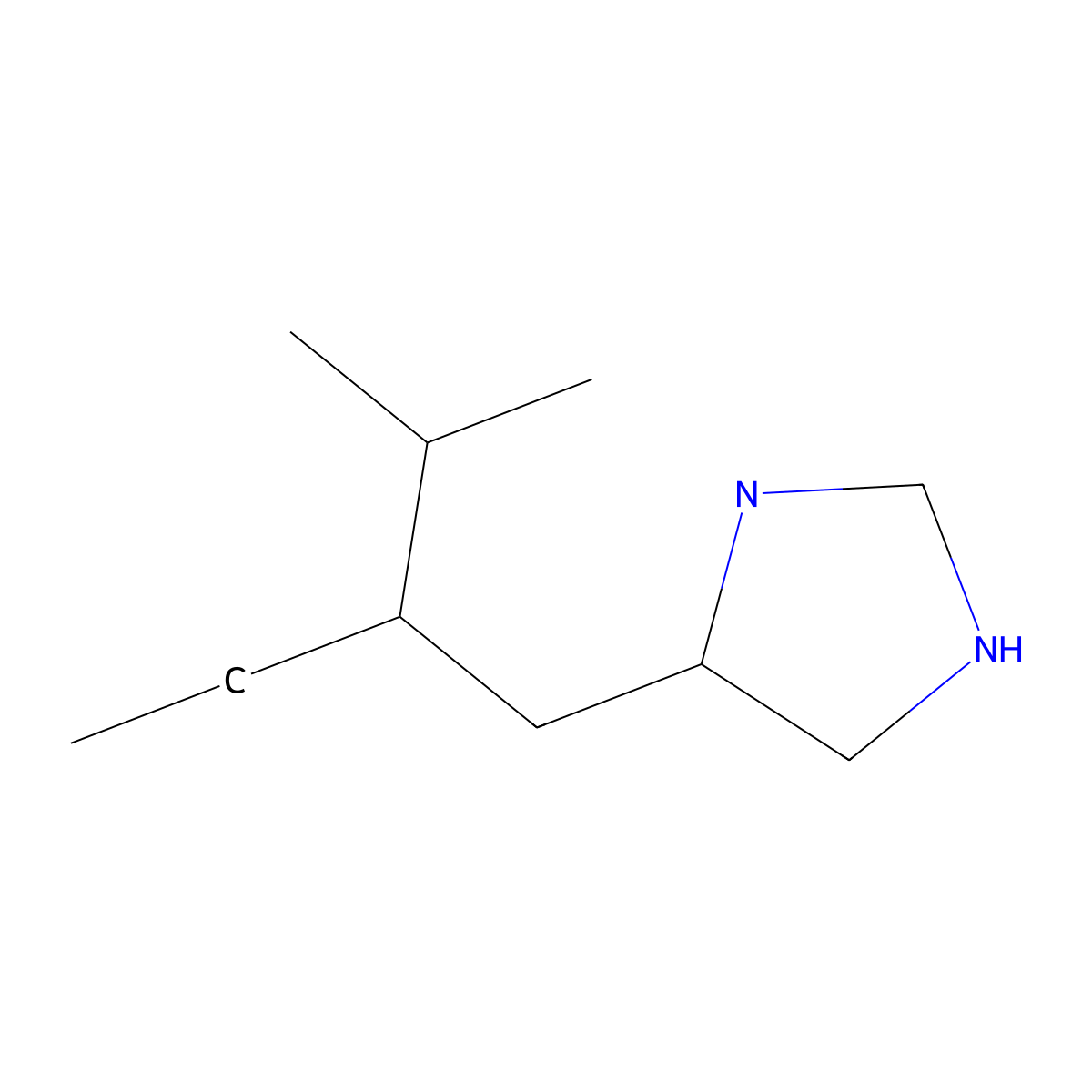} & 
        41.3 & 
        63.3 & 
        \includegraphics[width=0.12\textwidth]{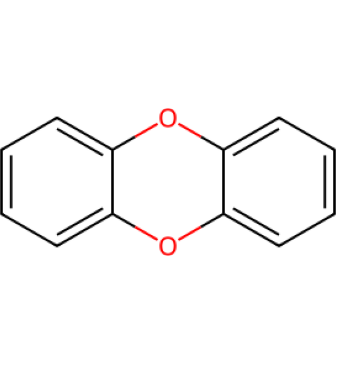} & 
        \includegraphics[width=0.12\textwidth]{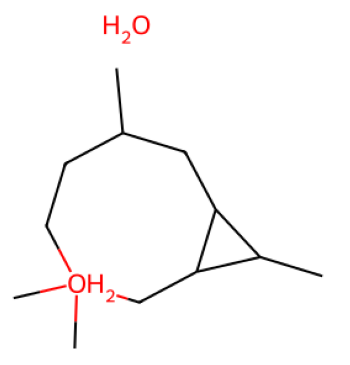} & 
        61.0 & 
        23.3 \\
        \includegraphics[width=0.12\textwidth]{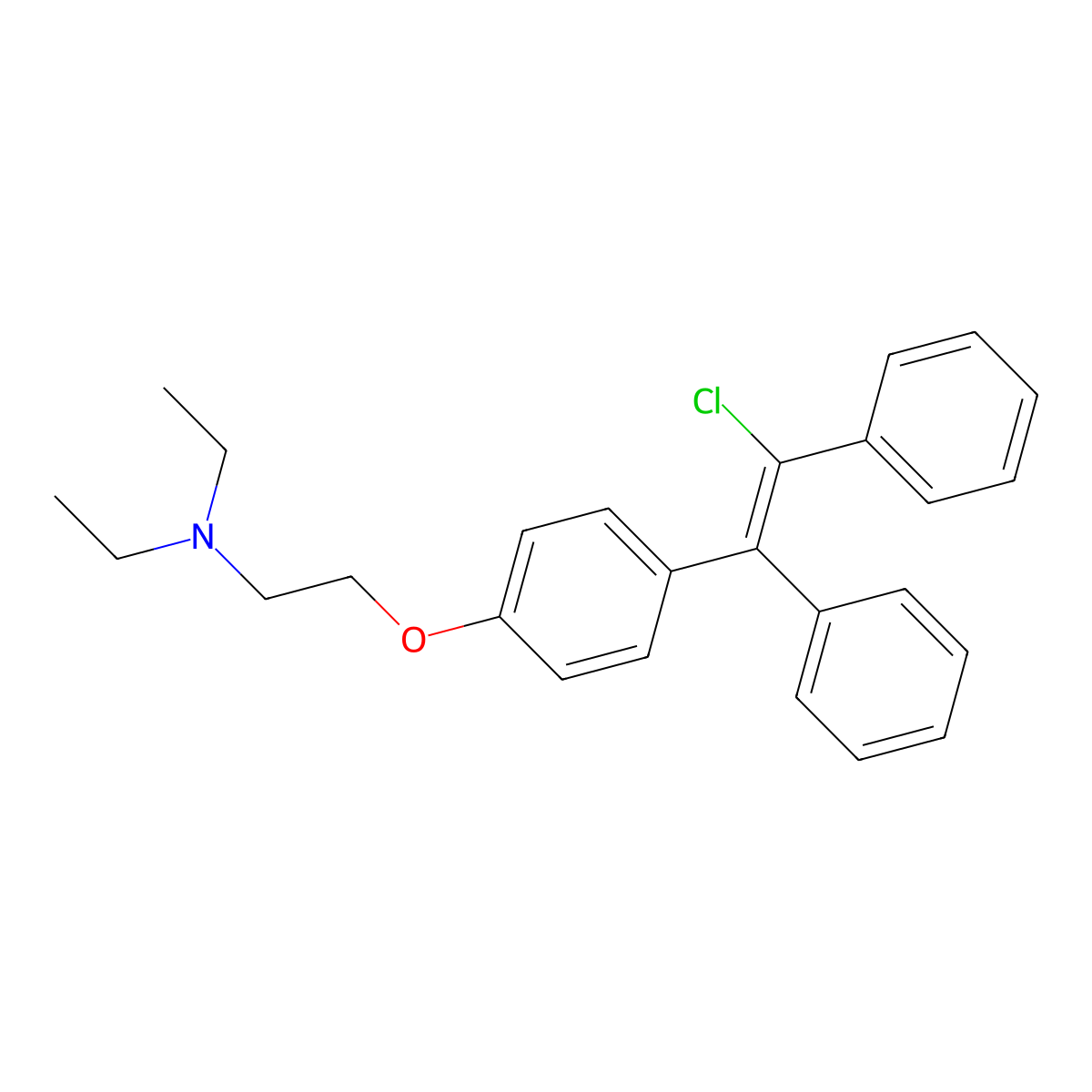} & 
        \includegraphics[width=0.12\textwidth]{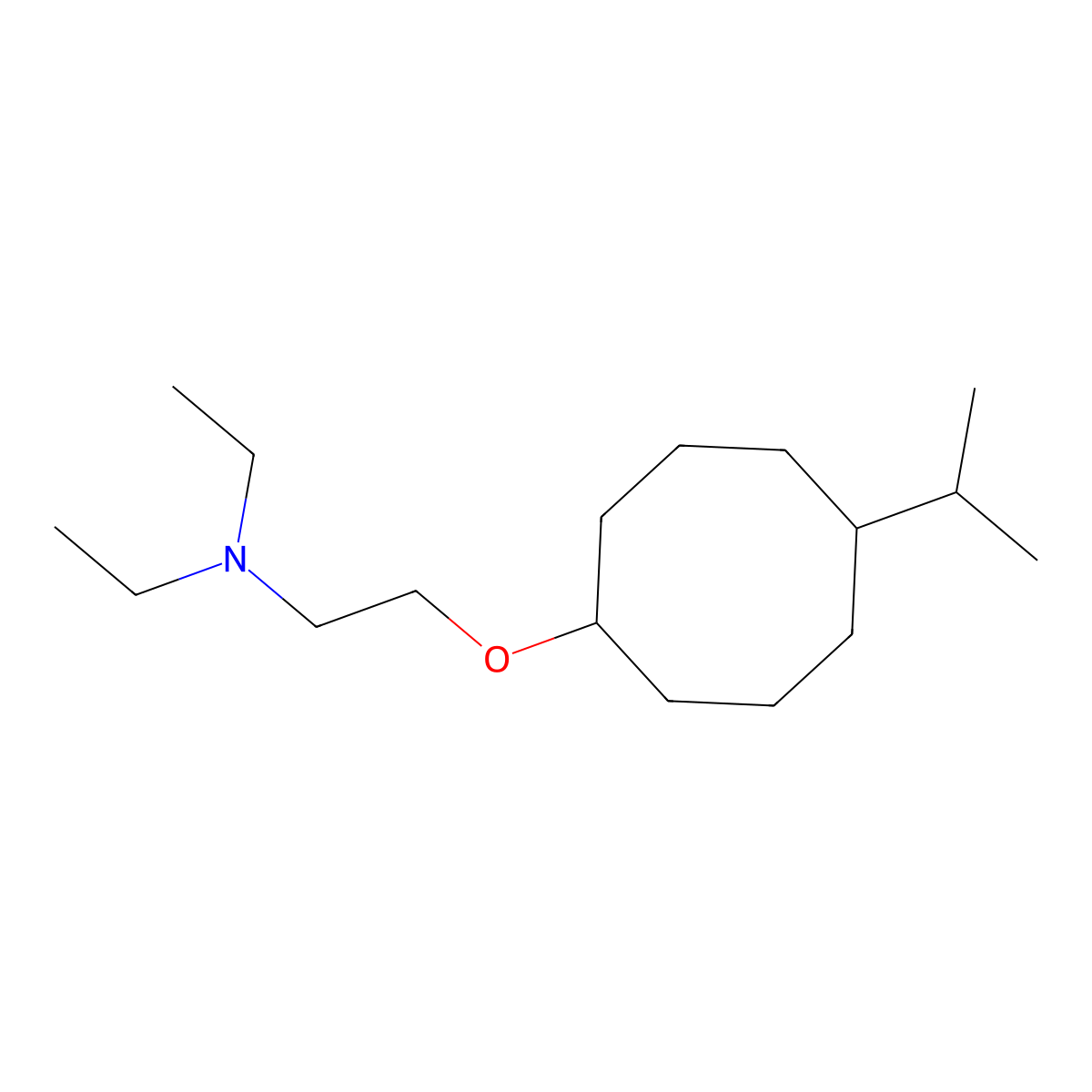} & 
        33.5 & 
        56.7 & 
        \includegraphics[width=0.12\textwidth]{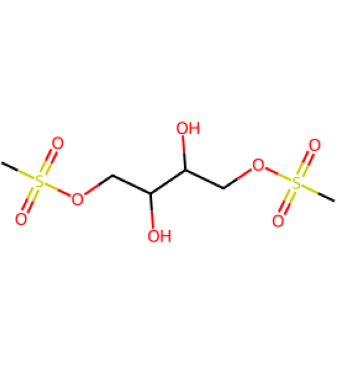} & 
        \includegraphics[width=0.12\textwidth]{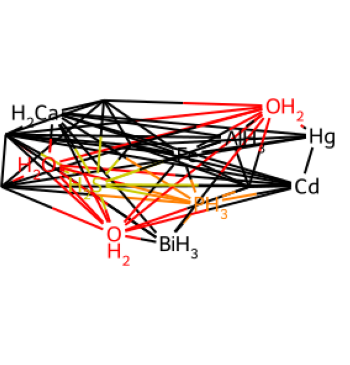} & 
        41.0 & 
        0.0 \\
        \hline
    \end{tabular}}
\end{table}

Additional examples of molecule reconstructions comparing \tool, DLG, and TabLeak are shown in \cref{fig:example_new}. In this set of examples, the first 3 columns show the exact reconstruction of the input. We also highlight that in cases where \tool does not managed to recover the entire graph, the attack can reconstruct subgraphs of the input (4th column), and a more realistic approximation otherwise (5th column).

\begin{figure}[h]
    \begin{center}
    	\includegraphics[width=\linewidth]{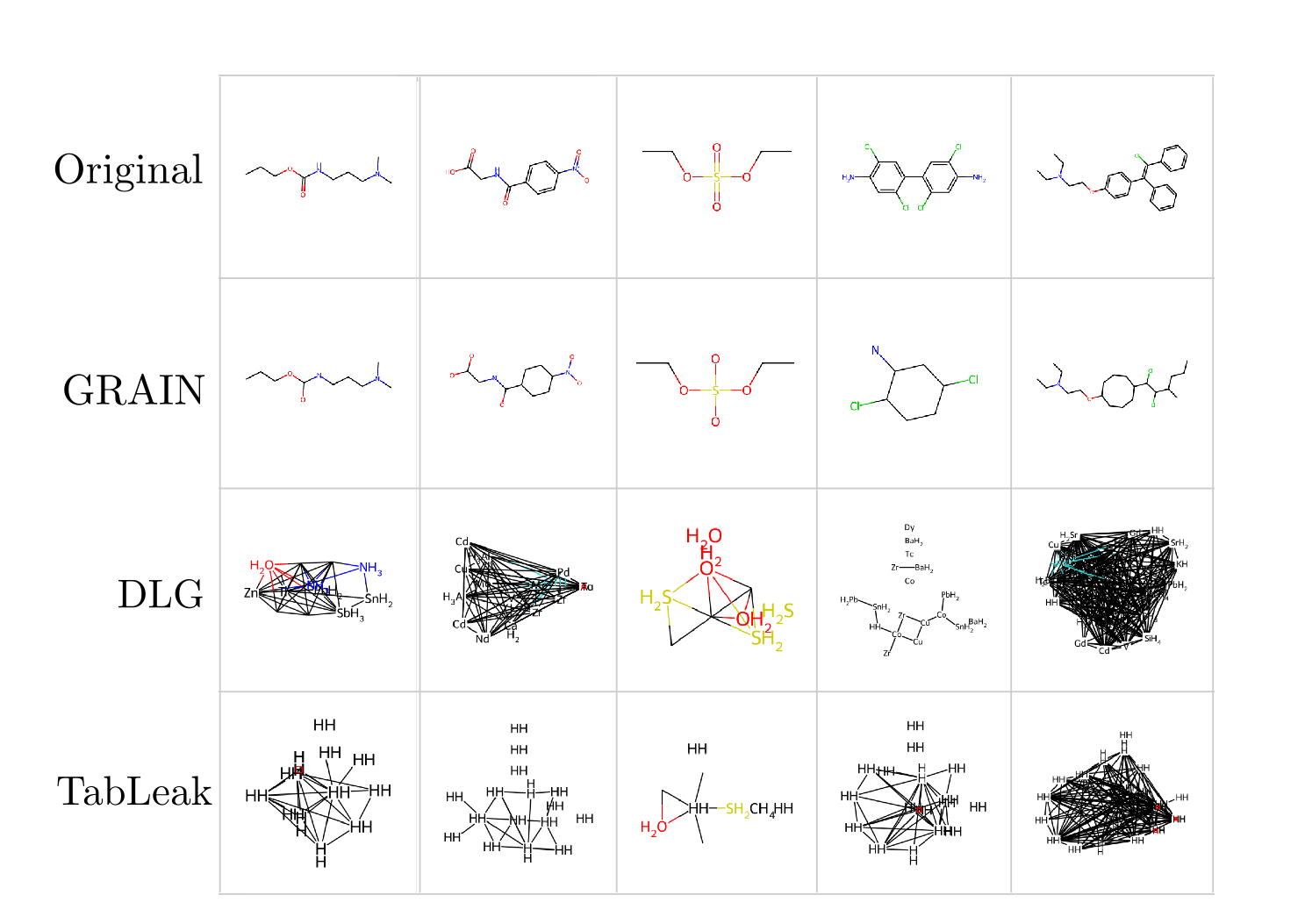}
	\end{center}
	\caption{Examples of molecule reconstructions compared between \tool, DLG, and TabLeak.}
	\label{fig:example_new}
	
\end{figure}

\subsection{Additional Chemical Datasets}
\label{app:chemical}

\begin{table*}[!hbt]\centering
	
	\caption{Results (in \%) of main experiments on 3 biochemical datasets -- Tox21, Clintox, BBBP. Here "$+A$" refers to the baseline attack with the input adjacency matrix given.} \label{table:chemical_main}
	\renewcommand{\arraystretch}{1.2}
	
	\newcommand{\temp}[1]{\textcolor{red}{#1}}
	\newcommand{\noopcite}[1]{} 
	
	\newcommand{\skiplen}{0.01\linewidth} 
	\newcommand{\rlen}{0.01\linewidth} 
	\newcolumntype{R}{>{$}r<{$}}
	\begin{tabular}{@{}c l RRRRRR @{}} \toprule
			&& \multicolumn{1}{c}{\text{\metric{0}}} & \multicolumn{1}{c}{\text{\metric{1}}}  & \multicolumn{1}{c}{\text{\metric{2}}} & \multicolumn{1}{c}{\text{FULL}} & \multicolumn{1}{c}{\text{Runtime[h]}}
			\\
			\midrule
			\multirow{5}{*}{Tox21}

			& \tool & \mathbf{86.9^{+4.2}_{-5.7}} & \mathbf{83.9^{+5.2}_{-6.9}} & \mathbf{82.6^{+5.7}_{-7.4}} &  \mathbf{68.0\pm1.7} & 14.3 \\ 
			& DLG & 31.8^{+4.5}_{-4.3} & 20.3^{+5.5}_{-4.8} & 22.8^{+6.6}_{-5.6} &  1.0\pm0.2 & 3.3\\
			& DLG $+A$ & 54.7^{+3.9}_{-4.2} & 60.1^{+4.6}_{-5.2} & 76.7^{+3.6}_{-4.8} &  1.0\pm0.2 & \mathbf{3.1}  \\
			& TabLeak & 25.1^{+5.1}_{-4.3} & 12.4^{+5.5}_{-4.3} & 10.8^{+5.6}_{-3.9} &  1.0\pm0.2 & 13.1
			\\
			& TabLeak $+A$ & 55.6^{+3.9}_{-3.9} & 57.7^{+4.1}_{-4.6} & 73.8^{+2.8}_{-3.5} & 1.0\pm0.2 & 12.3
			\\
			
			\midrule
			\multirow{5}{*}{Clintox}
			
			& \tool & \mathbf{73.7^{+5.7}_{-6.5}} & \mathbf{68.4^{+6.7}_{-7.8}} & 66.8^{+7.0}_{-7.6} &  \mathbf{36.0\pm1.2} & 24.1
			\\ 

			& DLG & 24.0^{+4.1}_{-3.8} & 10.3^{+4.8}_{-3.6} & 12.2^{+5.5}_{-4.2} & 
			 1.0\pm0.2 & 3.5\\
			& DLG $+A$ & 52.5^{+3.2}_{-3.6} & 52.6^{+4.1}_{-4.7} & \mathbf{72.3^{+3.2}_{-3.9}} &  1.0\pm0.2 & \mathbf{3.2}\\
			& TabLeak & 17.6^{+3.7}_{-2.8} & 6.0^{+4.0}_{-2.4} & 5.4^{+4.2}_{-2.5} & 1.0\pm0.2 & 15.2\\
			& TabLeak $+A$ & 54.0^{+3.4}_{-3.3} & 52.0^{+3.8}_{-4.2} & 62.8^{+3.3}_{-4.2} & 1.0\pm0.2 & 14.5  \\

			\midrule
			\multirow{5}{*}{BBBP}

			& \tool & \mathbf{71.7^{+5.9}_{-6.8}} & \mathbf{66.8^{+6.9}_{-7.7}} & 64.9^{+7.2}_{-8.0} & \mathbf{38.0\pm1.2} & 23.7 \\ 
			& DLG & 22.6^{+3.6}_{-3.3} & 8.8^{+4.9}_{-3.2} & 10.0^{+5.3}_{-3.7} &  0.0\pm0.0 & 3.9\\
			& DLG $+A$ & 51.6^{+3.1}_{-3.6} & 50.1^{+3.8}_{-4.5} & 70.6^{+3.1}_{-4.2} &  0.0\pm0.0 & \mathbf{3.1}\\
			& TabLeak & 17.6^{+3.8}_{-2.8} & 6.3^{+3.8}_{-2.5} & 4.7^{+3.7}_{-2.3} &  0.0\pm0.0 & 12.6
			\\
			& TabLeak $+A$ & 59.1^{+3.1}_{-3.6} & 59.4^{+3.6}_{-4.3} & \mathbf{71.9^{+2.9}_{-4.0}} & 0.0\pm0.0 & 12.5
			\\
			
			\bottomrule
	\end{tabular}
	
\end{table*}

In this section, we present our results on additional chemical datasets, namely Clintox and BBBP~\citep{wu2018moleculenet}. We highlight in \cref{table:chemical_main} that \tool generalizes across all settings, retaining its increased performance over the baseline attacks. We reaffirm that we achieve these results despite running for time comparable to the one of Tableak. 

\subsection{Additional Ablation Studies}
\label{app:ablation}
We perform additional ablation studies on various assumptions and parameters, demonstrating their effects on \tool. In particular, we observe how architectural changes might affect our performance, or how properties of the data might influence reconstructability.

\begin{table*}[!hbt]\centering
	
	\caption{Results (in \%) of main experiments with the LBFGS optimizer. Here "$+A$" refers to the baseline attack with the input adjacency matrix given.} \label{table:main_vs_baseline_lbfgs}
	\renewcommand{\arraystretch}{1.2}
	
	\newcommand{\temp}[1]{\textcolor{red}{#1}}
	\newcommand{\noopcite}[1]{} 
	
	\newcommand{\skiplen}{0.01\linewidth} 
	\newcommand{\rlen}{0.01\linewidth} 
	\newcolumntype{R}{>{$}r<{$}}
	\resizebox{\linewidth}{!}{
		\begin{tabular}{@{}c l RRRRR p{\skiplen} RRRRR @{}} \toprule
			&& \multicolumn{5}{c}{\text{GCN}} && \multicolumn{5}{c}{\text{GAT}}\\
			\cmidrule(l{5pt}r{5pt}){3-7} \cmidrule(l{5pt}r{5pt}){9-13} 
			&& \multicolumn{1}{c}{\text{\metric{0}}} & \multicolumn{1}{c}{\text{\metric{1}}}  & \multicolumn{1}{c}{\text{\metric{2}}} & \multicolumn{1}{c}{\text{FULL}} & \multicolumn{1}{c}{\text{Min/Rec}}
			&& \multicolumn{1}{c}{\text{\metric{0}}} & \multicolumn{1}{c}{\text{\metric{1}}}  & \multicolumn{1}{c}{\text{\metric{2}}} & \multicolumn{1}{c}{\text{FULL}} & \multicolumn{1}{c}{\text{Min/Rec}} \\ \midrule
			
			\multirow{5}{*}{CiteSeer}
			&\tool & 62.5^{+7.7}_{-8.2} & \mathbf{31.0^{+8.0}_{-7.8}} & \mathbf{31.6^{+8.1}_{-8.1}} & \mathbf{20.0\pm{0.8}} & \mathbf{1.5} && \mathbf{79.3^{+4.7}_{-6.3}} & \mathbf{69.1^{+6.1}_{-6.4}} & \mathbf{69.6^{+6.2}_{-6.0}} & \mathbf{61.0\pm{1.6}} & \mathbf{0.8}\\
			& DLG & \mathbf{67.7^{+3.9}_{-3.7}} & 0.0^{+0.0}_{-0.0} & 0.3^{+0.5}_{-0.3} & 0.0\pm{0.0} & 24.8 && 67.7^{+3.9}_{-3.7} & 0.0^{+0.0}_{-0.0} & 0.0^{+0.0}_{-0.0} & 0.0\pm{0.0} & 31.0\\
			& DLG $+A$ & \mathbf{67.7^{+4.1}_{-3.7}} & 0.0^{+0.0}_{-0.0} & 0.0^{+0.0}_{-0.0} & 0.0\pm{0.0} & 29.9 && 67.7^{+4.0}_{-3.7} & 0.0^{+0.0}_{-0.0} & 0.0^{+0.0}_{-0.0} & 0.0\pm{0.0} & 27.7\\
			& TabLeak & \mathbf{67.7^{+3.9}_{-3.7}} & 0.0^{+0.0}_{-0.0} & 0.0^{+0.0}_{-0.0} & 0.0\pm{0.0} & 158.8 && 67.7^{+3.9}_{-3.8} & 0.0^{+0.0}_{-0.0} & 0.0^{+0.0}_{-0.0} & 0.0\pm{0.0} & 153.0\\
			& TabLeak $+A$ & \mathbf{67.7^{+4.0}_{-3.7}} & 0.0^{+0.0}_{-0.0} & 0.0^{+0.0}_{-0.0} & 0.0\pm{0.0} & 202.0 && 67.7^{+4.0}_{-3.7} & 0.0^{+0.0}_{-0.0} & 0.0^{+0.0}_{-0.0} & 0.0\pm{0.0} & 148.7\\
			
			\midrule
			\multirow{5}{*}{Pokec}
			&\tool&\mathbf{58.3^{+5.9}_{-5.9}} & \mathbf{50.7^{+7.9}_{-7.8}} & \mathbf{55.8^{+8.3}_{-7.9}} & \mathbf{15.0\pm{0.8}} & \mathbf{0.1} && \mathbf{97.2^{+1.6}_{-1.9}} & \mathbf{93.5^{+3.4}_{-4.2}} & \mathbf{96.3^{+1.9}_{-2.3}} & \mathbf{79.0\pm{1.8}} & \mathbf{0.2}\\
			& DLG&44.6^{+6.8}_{-6.2} & 11.3^{+16.0}_{-11.3} & 13.7^{+20.1}_{-13.7} & 0.0\pm{0.0} & 37.8 &&44.7^{+2.3}_{-2.3} & 2.2^{+3.1}_{-2.2} & 0.0^{+0.0}_{-0.0} & 0.0\pm{0.0} & 26.3\\
			& DLG $+A$&48.7^{+12.7}_{-8.6} & 39.1^{+18.7}_{-16.9} & 51.8^{+15.9}_{-17.8} & 1.0\pm{0.2} & 38.5 &&57.4^{+3.7}_{-3.9} & 69.5^{+3.6}_{-4.0} & 88.6^{+2.0}_{-2.1} & 0.0\pm{0.0} & 21.6\\
			& TabLeak&49.6^{+8.7}_{-6.6} & 8.2^{+12.0}_{-8.2} & 5.6^{+9.3}_{-5.6} & 0.0\pm{0.0} & 177.5 && 50.8^{+12.4}_{-8.9} & 13.9^{+13.5}_{-12.3} & 7.9^{+11.9}_{-7.9} & 0.0\pm{0.0} & 204.5\\
			& TabLeak $+A$&49.9^{+4.1}_{-4.3} & 38.1^{+5.8}_{-5.9} & 58.9^{+6.1}_{-6.6} & 0.0\pm{0.0} & 216.0&&52.6^{+3.3}_{-3.3} & 68.1^{+4.1}_{-3.9} & 82.7^{+4.0}_{-4.9} & 0.0\pm{0.0} & 254.5\\
			
			\bottomrule
		\end{tabular}
	}
	
\end{table*}
\paragraph{Effect of optimizer on baseline results}
In our main experiments we presented results for the CiteSeer and Pokec datasets with the baselines running an SGD optimizer. In \cref{table:main_vs_baseline_lbfgs}, we present results with the more stable LBFGS optimizer averaged across 10 reconstructions due to time limits. We see that the baselines show better performance, however, \tool still outperforms them and is less resource-consuming, requiring up to $100\times$ less runtime.

\begin{table*}[!hbt]\centering
	
	\caption{Results (in \%) of \tool and the baselines in cases of different model parameters. Here $L$ is the number of GCN layers and $d'$ is the model's width. $L=2, d'=300$ is the original setting. } \label{table:network_sizes_table}
	\renewcommand{\arraystretch}{1.2}
	
	\newcommand{\temp}[1]{\textcolor{red}{#1}}
	\newcommand{\noopcite}[1]{} 
	
	\newcommand{\skiplen}{0.01\linewidth} 
	\newcommand{\rlen}{0.01\linewidth} 
	\newcolumntype{R}{>{$}r<{$}}

		\begin{tabular}{@{}c l RRRRRR @{}} \toprule
			
			&& \multicolumn{1}{c}{\text{\metric{0}}} & \multicolumn{1}{c}{\text{\metric{1}}}  & \multicolumn{1}{c}{\text{\metric{2}}} & \multicolumn{1}{c}{\text{FULL}}  \\ 
			
			\midrule

			\multirow{5}{*}{\shortstack{$L = 2,$ \\ $d'=300$ \\ (default)}}

			& \tool & \mathbf{86.9^{+4.2}_{-5.7}} & \mathbf{83.9^{+5.2}_{-6.9}} & \mathbf{82.6^{+5.7}_{-7.4}} & \mathbf{68.0\pm1.7} \\ 
			
			& DLG & 31.8^{+4.5}_{-4.3} & 20.3^{+5.5}_{-4.8} & 22.8^{+6.6}_{-5.6} & 1.0\pm0.2\\
			
			& DLG $+A$ & 54.7^{+3.9}_{-4.2} & 60.1^{+4.6}_{-5.2} & 76.7^{+3.6}_{-4.8} &  1.0\pm0.2  \\
			
			& TabLeak & 25.1^{+5.1}_{-4.3} & 12.4^{+5.5}_{-4.3} & 10.8^{+5.6}_{-3.9} & 1.0\pm0.2 \\
			
			& TabLeak $+A$ & 55.6^{+3.9}_{-3.9} & 57.7^{+4.1}_{-4.6} & 73.8^{+2.8}_{-3.5} & 
			1.0\pm0.2 \\
			
			\midrule
			\multirow{5}{*}{\shortstack{$L = 3,$ \\ $d'=300$}}
			
			& \tool & \mathbf{82.5^{+5.7}_{-7.7}} & \mathbf{80.7^{+6.3}_{-7.7}} & \mathbf{80.4^{+6.2}_{-7.8}} &  \mathbf{63.0\pm1.6}
			\\ 
			& DLG & 20.3^{+4.3}_{-3.4} & 7.8^{+5.1}_{-3.3} & 8.2^{+5.3}_{-3.4} & 1.0\pm0.2  \\
			
			& DLG $+A$ & 43.0^{+3.7}_{-3.6} & 48.0^{+4.3}_{-4.5} & 66.0^{+3.7}_{-4.6} & 1.0\pm0.2
			\\
			& TabLeak & 16.5^{+3.8}_{-2.9} & 8.8^{+4.4}_{-3.1} & 8.0^{+4.3}_{-3.0} & 1.0\pm0.2
			\\
			& TabLeak $+A$ & 47.5^{+4.0}_{-4.2} & 48.1^{+4.8}_{-5.0} & 62.9^{+4.3}_{-4.4} & 1.0\pm0.2 \\
			
			\midrule
			\multirow{5}{*}{\shortstack{$L = 4,$ \\ $d'=300$}}

			& \tool & \mathbf{83.9^{+5.5}_{-7.4}} & \mathbf{82.8^{+5.9}_{-7.7}} & \mathbf{82.8^{+6.0}_{-7.9}} & \mathbf{64.0\pm1.6} \\ 
			
			& DLG & 14.1^{+3.8}_{-2.8} & 4.0^{+4.7}_{-2.2} & 4.8^{+4.9}_{-2.6} & 1.0\pm0.2 \\
			
			& DLG $+A$ & 39.1^{+3.7}_{-3.8} & 37.0^{+5.3}_{-5.4} & 55.6^{+5.0}_{-5.7} & 1.0\pm0.2 \\
			
			& TabLeak & 12.0^{+3.4}_{-1.9} & 2.1^{+4.3}_{-1.4} & 3.4^{+4.0}_{-1.7} & 1.0\pm0.2 \\
			
			& TabLeak $+A$ & 30.0^{+4.7}_{-4.0} & 27.3^{+5.9}_{-5.1} & 51.1^{+4.9}_{-5.3} & 1.0\pm0.2 \\
			
			\midrule
			\multirow{5}{*}{\shortstack{$L = 2,$ \\ $d'=200$}}
			
			& \tool & \mathbf{84.6^{+4.6}_{-6.4}} & \mathbf{81.4^{+5.8}_{-6.9}} & \mathbf{80.5^{+5.9}_{-7.2}} &  \mathbf{62.0\pm1.6}
			\\ 
			& DLG & 30.8^{+4.5}_{-4.1} & 18.9^{+5.8}_{-4.9} & 22.2^{+6.7}_{-5.4} &  1.0\pm0.2
			\\

			& DLG $+A$ & 50.3^{+4.2}_{-4.2} & 53.4^{+5.3}_{-5.9} & 68.7^{+4.9}_{-6.1} & 3.0\pm0.4

			\\
			& TabLeak & 22.1^{+4.8}_{-3.7} & 10.3^{+5.3}_{-3.6} & 8.9^{+5.5}_{-3.6} & 1.0\pm0.2
			\\

			& TabLeak $+A$ & 55.0^{+4.8}_{-5.0} & 62.1^{+4.9}_{-5.9} & 76.7^{+3.6}_{-4.7} & 1.0\pm0.2\\

			\midrule
			\multirow{5}{*}{\shortstack{$L = 2,$ \\ $d'=400$}}
			
			& \tool & \mathbf{85.2^{+4.6}_{-6.1}} & \mathbf{81.5^{+5.4}_{-7.1}} & \mathbf{80.1^{+6.1}_{-7.5}} &  \mathbf{63.0\pm1.6}
			\\ 
			& DLG & 35.1^{+4.9}_{-4.7} & 26.1^{+6.4}_{-5.6} & 25.0^{+6.9}_{-6.0} &  1.0\pm0.2
			\\

			& DLG $+A$ & 57.6^{+3.9}_{-4.3} & 61.7^{+4.7}_{-5.5} & 72.5^{+4.3}_{-5.5} &  2.0\pm0.3
			\\
			
			& TabLeak & 28.5^{+4.5}_{-4.0} & 17.1^{+5.4}_{-4.4} & 12.9^{+5.4}_{-4.0} &  1.0\pm0.2
			\\

			& TabLeak $+A$ & 61.7^{+3.6}_{-3.7} & 62.6^{+3.6}_{-4.4} & 76.3^{+2.9}_{-3.3} & 1.0\pm0.2
			\\

			\bottomrule
		\end{tabular}

\end{table*}
\paragraph{Effect of model parameters on reconstruction quality}
First, in \cref{table:network_sizes_table} and \cref{tab:dim_ablation} we demonstrate the performance of \tool under modifying the model parameters. We observe that neither changing in the number of layers nor the hidden dimension size of the GCN substantially affects the performance of \tool, while reaffirming the significant improvement over the baselines, even when they are given the graph connections as prior knowledge. We note that we only utilise the first 2 GCN layers even when $L>2$, showing the robustness of our method.

Additionally, we note that \tool is not significantly impacted by the embedding dimension $d^\prime$, as long as $n < d^\prime$, consequently achieving similar scores, particularly for small graphs. We show the exact results in \cref{tab:dim_ablation}.

\begin{table*}[!hbt]\centering
        
    \caption{Results (in \%) of \tool with different embedding dimensions across a range of graph sizes}
    \label{tab:dim_ablation}
    \newcommand{\twocol}[1]{\multicolumn{2}{c}{#1}}
    \newcommand{\threecol}[1]{\multicolumn{3}{c}{#1}}
    \newcommand{\fivecol}[1]{\multicolumn{5}{c}{#1}}
    \newcommand{\ninecol}[1]{\multicolumn{9}{c}{#1}}
    
    \newcommand{\bsz}{Batch Size~}
    \newcommand{\certified}{{CR(\%)}}
    
    \renewcommand{\arraystretch}{1.2}
    
    \newcommand{\ccellt}[2]{\colorbox{#1}{\makebox(20,8){{#2}}}}
    \newcommand{\ccellc}[2]{\colorbox{#1}{\makebox(8,8){{#2}}}}
    \newcommand{\ccells}[2]{\colorbox{#1}{\makebox(55,8){{#2}}}}
    
    \newcommand{\temp}[1]{\textcolor{red}{#1}}
    \newcommand{\noopcite}[1]{} 
    
    \newcommand{\skiplen}{0.01\linewidth} 
    \newcommand{\rlen}{0.01\linewidth} 
    \newcolumntype{R}{>{$}r<{$}}
    \resizebox{\linewidth}{!}{
        \begin{tabular}{@{}l RRR p{\skiplen}  RRR p{\skiplen} RRR @{}} \toprule
            & \threecol{$n\leq15$} && \threecol{$16\leq n\leq 25$} && \threecol{$26\leq n$}\\
            
            \cmidrule(l{5pt}r{5pt}){2-4} \cmidrule(l{5pt}r{5pt}){6-8} \cmidrule(l{5pt}r{5pt}){10-12} 
            
            & \multicolumn{1}{c}{\text{\metric{0}}} & \multicolumn{1}{c}{\text{\metric{2}}} & \multicolumn{1}{c}{\text{FULL}}
            && \multicolumn{1}{c}{\text{\metric{0}}} & \multicolumn{1}{c}{\text{\metric{2}}} & \multicolumn{1}{c}{\text{FULL}}
            && \multicolumn{1}{c}{\text{\metric{0}}} & \multicolumn{1}{c}{\text{\metric{2}}} & \multicolumn{1}{c}{\text{FULL}}\\ 
            \midrule			
            \multirow{1}{*}{$d=300$} 
            & \mathbf{93.0^{+3.4}_{-5.4}} & \mathbf{91.6^{+3.8}_{-6.3}} &  \mathbf{81.9\pm1.7} &&
            \mathbf{81.7^{+3.9}_{-4.8}} & \mathbf{74.8^{+5.8}_{-6.3}} &  \mathbf{43.6\pm1.1} &&
            50.1^{+6.8}_{-7.1} & 39.2^{+8.5}_{-7.7} &  \mathbf{5.1\pm0.6} \\
            \midrule
            
            \multirow{1}{*}{$d=128$} 
            & \mathbf{92.1^{+3.2}_{-5.0}} & \mathbf{92.3^{+3.9}_{-5.7}} &  \mathbf{79.3\pm{1.6}} &&
            \mathbf{81.4^{+4.0}_{-4.8}} & \mathbf{75.1^{+5.7}_{-6.6}} &  \mathbf{43.6\pm1.1} &&
            49.3^{+7.2}_{-6.5} & 38.8^{+8.7}_{-7.6} &  \mathbf{5.1\pm0.6} \\

            \multirow{1}{*}{$d=64$} 
            & \mathbf{92.2^{+3.0}_{-5.5}} & \mathbf{92.0^{+4.0}_{-5.9}} &  \mathbf{79.3\pm{1.6}} &&
            \mathbf{81.3^{+4.1}_{-4.7}} & \mathbf{75.5^{+5.8}_{-6.5}} &  \mathbf{43.6\pm1.1} &&
            48.6^{+7.4}_{-6.5} & 37.9^{+9.0}_{-7.7} &  \mathbf{5.1\pm0.6} \\

            \multirow{1}{*}{$d=32$} 
            & \mathbf{92.2^{+3.0}_{-5.5}} & \mathbf{91.7^{+3.6}_{-6.5}} &  \mathbf{79.3\pm{1.6}} &&
            \mathbf{81.7^{+4.0}_{-4.4}} & \mathbf{73.8\pm6.1} &  \mathbf{43.6\pm1.1} &&
            15.3^{+2.8}_{-4.4} & 13.3^{+2.5}_{-3.9} &  \mathbf{0.0\pm0.0} \\
        
            \bottomrule
        \end{tabular}
    }
    
\end{table*}

\paragraph{Effect of span check threshold on filtering capabilities}
We now investigate the effect of the choice for the $\tau$ threshold, used for filtering inputs using the span check method. We measure the ratio between the number of nodes and 1-hop building blocks that pass the filter, and the actual number of these blocks. We explore different values of $\tau$ in the range $[10^{-6}, 1]$, and evaluate this metric on 10 randomly chosen samples from the Tox21 dataset. We show in \cref{fig:tau_ablation} that any $\tau \in [10^{-4}, 10^{-2}]$ results in essentially the same filtering results, and that thresholds in this interval perfectly recover the correct 1-hop building blocks.

\begin{wrapfigure}{R}{0.4\textwidth}
	  \vspace{-10mm}
	\begin{minipage}{.4\textwidth}
        \begin{figure}[H]
            \centering
            \caption{Ablation study on the span check filtering threshold $\tau$.}
            \label{fig:tau_ablation}
            \includegraphics[width=\linewidth]{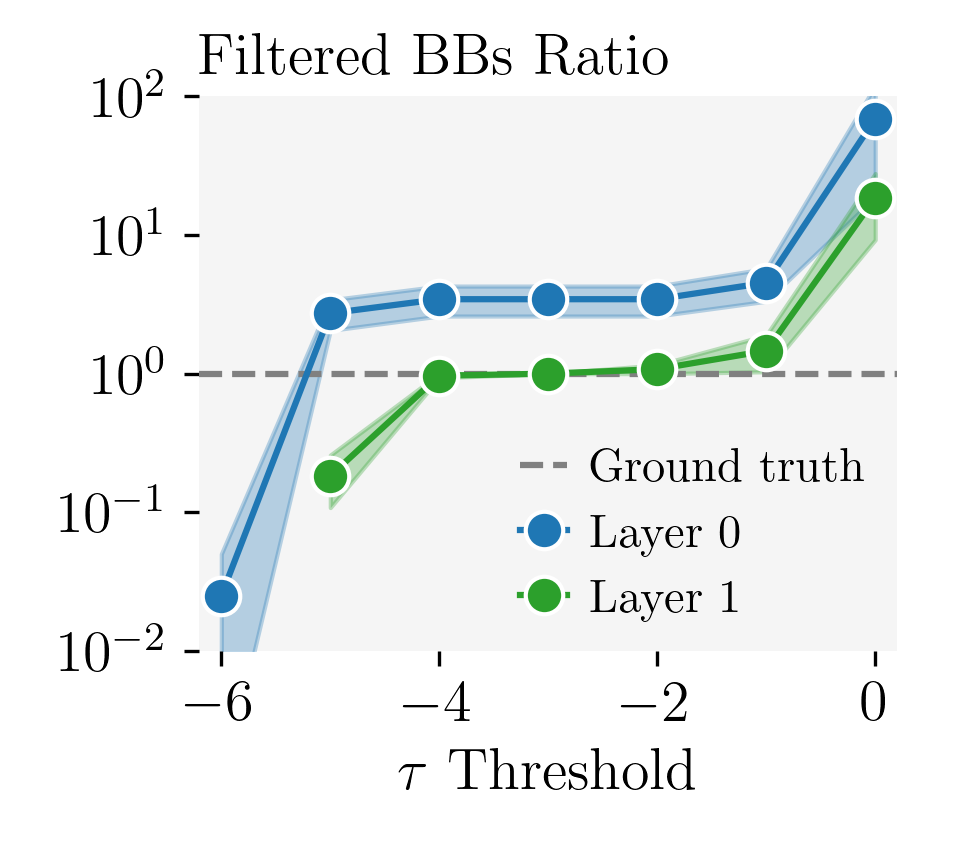}
        \end{figure}        
    \end{minipage}
     \vspace{-10mm}
\end{wrapfigure}
\paragraph{Adjacency matrix low-rankness effect on reconstructability}
Further, in \cref{fig:a_full_rankness} we looked into how the rank-deficiency of the adjacency matrix $A$ affects how much of the input \tool might be able to recover. For different sizes of $A$, we measure what the Monte-Carlo probability of $A$ being full-rank, and the fraction of nodes we can recover, as computed per \cref{thm:spancheck_ours}. This was done for synthetic graphs, where we sampled 100,000 symmetric binary matrices with varying probability of every 2 nodes being connected, as well as for all molecular graphs in the chemical datasets Clintox, Tox21 and BBBP. We show that \cref{thm:spancheck_ours} is crucial for understanding why \tool is effective, despite the probability of $A$ being full-rank being low. In particular, we highlight in \cref{fig:a_full_rankness} that \tool can recover an increasing fraction of nodes as $A$ grows.

\begin{table*}[t]\centering
	
	\caption{Results (in \%) of \tool tested on the Pokec social network dataset ~\citep{pokec}. 20 subgraphs were sampled for each of the size ranges 25-30, 30-40, 40-50 and 50-60} 
	\label{tab:pokec}
	\newcommand{\twocol}[1]{\multicolumn{2}{c}{#1}}
	\newcommand{\threecol}[1]{\multicolumn{3}{c}{#1}}
	\newcommand{\fivecol}[1]{\multicolumn{5}{c}{#1}}
	\newcommand{\ninecol}[1]{\multicolumn{9}{c}{#1}}
	
	\newcommand{\bsz}{Batch Size~}
	\newcommand{\certified}{{CR(\%)}}
	
	\renewcommand{\arraystretch}{1.2}
	
	\newcommand{\ccellt}[2]{\colorbox{#1}{\makebox(20,8){{#2}}}}
	\newcommand{\ccellc}[2]{\colorbox{#1}{\makebox(8,8){{#2}}}}
	\newcommand{\ccells}[2]{\colorbox{#1}{\makebox(55,8){{#2}}}}
	
	\newcommand{\temp}[1]{\textcolor{red}{#1}}
	\newcommand{\noopcite}[1]{} 
	
	\newcommand{\skiplen}{0.01\linewidth} 
	\newcommand{\rlen}{0.01\linewidth} 
	\newcolumntype{C}{>{\centering\arraybackslash$}c<{$}}

		\begin{tabular}{@{}lCCCCC @{}} \toprule
			
			\multicolumn{1}{c}{$n$} & 
            \multicolumn{1}{c}{\text{\metric{0}}} & \multicolumn{1}{c}{\text{\metric{1}}} & \multicolumn{1}{c}{\text{\metric{2}}} &  \multicolumn{1}{c}{\text{FULL}} & 
            \multicolumn{1}{c}{\text{Runtime[h]}} \\ 
			
			\midrule
			
			\multirow{1}{*}{25-30} 
			& 98.3^{+0.2}_{-0.4} & 95.1^{+0.5}_{-1.1} & 96.8^{+0.4}_{-0.9} & 17/20 & 0.17\\
			
			\multirow{1}{*}{30-40} 
			& 83.1^{+2.3}_{-3.4} & 61.6^{+3.1}_{-3.0} & 79.4^{+2.7}_{-3.6} & 5/20 & 0.46\\
			
			\multirow{1}{*}{40-50}
			& 69.3^{+3.2}_{-3.8} & 38.0^{+4.7}_{-4.3} & 59.2^{+3.7}_{-4.0} & 2/20 & 0.64\\			
			
			\multirow{1}{*}{50-60}
			& 32.7^{+4.8}_{-3.9} & 23.3^{+4.2}_{-3.5} & 41.2^{+4.6}_{-4.1} & 3/20 & 0.43\\
            
            \midrule

            \multirow{1}{*}{Total}
			& 70.9^{+6.2}_{-6.5} & 55.6\pm 7.2 & 69.2^{+6.4}_{-6.6} & 27/80 & 1.70\\

			\bottomrule
		\end{tabular}
	
\end{table*}

\newpage
\section{Limitations}
\label{sec:limitations}

\tool is the first algorithm to advance the field of gradient inversion for graph data, and we see significant potential for further development. However, our attack method currently assumes that the FL protocol includes node degree as a node feature. While this assumption holds in many GNN settings, relaxing it leads to reduced performance. Improving performance under relaxed assumptions is an important direction for future research.

Another key area for improvement is reducing the computational complexity of \tool, enabling it to scale to larger graphs and graphs with nodes of higher degree. We believe this to be a promising future direction, as there are many potential optimizations that could enhance the algorithm's efficiency. Particularly promising is using data priors that have the potential to severely reduce the number of span checks that need to be performed by efficiently filtering impossible subgraphs.

Additionally, \tool relies on the assumption that the GNN architecture satisfies \cref{assmp:dependence}. While this assumption holds for widely used GNNs like GCN and GAT, adapting our algorithm to support other GNN types is left for future work.

Moreover, as discussed in \cref{sec:technical}, \tool requires that $n < d'$ to maintain the low-rank nature of gradient updates. While this is a limitation, we believe it applies to many real-world scenarios, meaning that practical FL settings remain vulnerable to privacy risks.

Another assumption \tool depends on is that the weighted adjacency matrix $\mA$ is high-rank, with full reconstruction possible only if $\mA$ is full-rank. While this holds for GAT architectures, it is not always true for GCNs as presented in \cref{fig:a_full_rankness}. Relaxing the full-rankness assumption would be a crucial step toward better understanding privacy risks in FL for GNNs.

Finally, we leave the investigation of potential defenses against \tool as well as more complex federated learning protocols such as Federated Averaging for future work.

\fi

\end{document}